\documentclass[11pt]{article}
\usepackage[a4paper, margin=1in]{geometry} 
\usepackage{mathtools}
\usepackage[colorlinks, linkcolor=blue,citecolor = green,urlcolor=cyan]{hyperref}
\usepackage[framemethod=tikz]{mdframed}
\usepackage{wrapfig}

\DeclarePairedDelimiter\floor{\lfloor}{\rfloor}
\usepackage{fullpage}
\usepackage[english]{babel}
\usepackage[utf8]{inputenc}
\usepackage{amsmath}
\usepackage{graphicx}
\usepackage{amsmath,amsthm,amssymb,epsfig,enumerate,graphicx}

\usepackage{algpseudocode}
\usepackage{algorithm}
\usepackage{algorithmicx}
\newtheorem{claim}{Claim}
\newtheorem{lemma}{Lemma}
\newtheorem{Observation}{Observation}
\newtheorem{definition}{Definition}
\newtheorem{corollary}{Corollary}
\newtheorem{theorem}{Theorem}
\newcommand{\RandBasic}{\mathsf{BasicRandTimer}}
\newcommand{\RandImprovedTimer}{\mathsf{ImprovedRandTimer}}
\newcommand{\DetTimer}{\mathsf{DetTimer}}

\newcommand{\DetCounter}{\mathsf{DetCounter}}
\newcommand{\ApproxCounter}{\mathsf{ApproxCounter}}

\newcommand{\Last}{\mathsf{Last}}

\newcommand{\cN}{{\mathcal N}}
\newcommand{\poly}{\mathsf{poly}}
\newcommand{\sync}{\mathsf{sync}}
\newcommand{\async}{\mathsf{async}}

\newcommand{\inn}{\mathsf{in}}
\newcommand{\outt}{\mathsf{out}}
\newcommand{\delaynode}{\mathsf{delay}}
\newcommand{\dec}{\mathsf{dec}}

\newcommand{\local}{${\mathsf{LOCAL}}$}
\newcommand{\congest}{${\mathsf{CONGEST}}$}
\renewcommand{\paragraph}[1]{\vspace{0.15cm}\noindent {\bf #1}}

\let\Pr\relax
\DeclareMathOperator*{\Pr}{\mathrm{Pr}}

\DeclareMathOperator{\Bias}{b}
\DeclareMathOperator{\pot}{pot}

\usepackage{wrapfig}
\usepackage{geometry}
  \usepackage{nth}
  \usepackage{intcalc}

\begin{document}

\newcommand{\yael}[1]{\textcolor{magenta}{Yael: #1}}
\newcommand{\merav}[1]{\textcolor{blue}{Merav: #1}}

\title{Counting to Ten with Two Fingers:\\  Compressed Counting with Spiking Neurons}
\author{Yael Hitron \and Merav Parter\footnote{Department of Computer 
		Science and Applied Mathematics, Weizmann Institute of Science, Rehovot 76100,
		Israel. Emails:
		\texttt{\{yael.hitron,merav.parter\}@weizmann.ac.il}.}
	\thanks{Supported
		in part by the BSF-NSF grants.}}
\date{}

\maketitle

\begin{abstract}
We consider the task of measuring time with probabilistic threshold gates implemented by bio-inspired spiking neurons. 
In the model of \emph{spiking neural networks}, network evolves in discrete rounds, where in each round,  neurons fire in pulses in response to a sufficiently high membrane potential. This potential is induced by spikes from neighboring neurons that fired in the previous round, which can have either an excitatory or inhibitory effect.

Discovering the underlying mechanisms by which the brain perceives the duration of time is one of the largest open enigma in computational neuro-science. To gain a better algorithmic understanding onto these processes, we introduce the \emph{neural timer} problem. In this problem, one is given a time parameter $t$, an input neuron $x$, and an output neuron $y$. It is then required to design a minimum sized neural network (measured by the number of auxiliary neurons) in which every spike from $x$ in a given round $i$, makes the output $y$ fire for the subsequent $t$ consecutive rounds.

We first consider a deterministic implementation of a neural timer and show that $\Theta(\log t)$ (deterministic) threshold gates are both sufficient and necessary. This raised the question of whether randomness can be leveraged to reduce the number of neurons. We answer this question in the affirmative by considering neural timers with spiking neurons where the neuron $y$ is required to fire for $t$ consecutive rounds with probability at least $1-\delta$, and should stop firing after at most $2t$ rounds with probability $1-\delta$ for some input parameter $\delta \in (0,1)$. Our key result is a construction of a neural timer with $O(\log\log 1/\delta)$ spiking neurons. Interestingly, this construction uses only \emph{one} spiking neuron, while the remaining neurons can be deterministic threshold gates. We complement this construction with a matching lower bound of $\Omega(\min\{\log\log 1/\delta, \log t\})$ neurons. This provides the first separation between deterministic and randomized constructions in the setting of spiking neural networks.

Finally, we demonstrate the usefulness of compressed counting networks for \emph{synchronizing} neural networks. In the spirit of distributed synchronizers [Awerbuch-Peleg, FOCS'90], we provide a general transformation (or simulation) that can take any synchronized network solution and simulate it in an asynchronous setting (where edges have arbitrary response latencies) while incurring a small overhead w.r.t the number of neurons and computation time. 
\end{abstract}


\newpage
\vspace{-10pt}\section{Introduction}
Understanding the mechanisms by which brain experiences time is one of the major research objectives in neuroscience \cite{merchant2013neural,allman2014properties,finnerty2015time}. Humans measure time using a global clock based on standardized units of minutes, days and years. In contrast, the brain perceives time using specialized neural clocks that define their own time units. Living organisms have various other implementations of biological clocks, a notable example is the circadian clock that gets synchronized with the rhythms of a day.

In this paper we consider the algorithmic aspects of measuring \emph{time} in a simple yet biologically plausible model of \emph{stochastic spiking neural networks} (SNN) \cite{maass1996computational,maass1997networks}, in which neurons fire in discrete pulses, in response to a sufficiently high membrane potential.  This model is believed to capture the spiking behavior observed in real neural networks, and has recently received quite a lot of attention in the algorithmic community \cite{LynchMP16,LynchMP-bda17,LynchMP17,lynch2018basic,Legenstein0PV18,PapadimitriouV19,ChouCL19}. In contrast to the common approach in computational neuroscience and machine learning, the focus here is not on general computation ability or broad learning tasks, but rather on specific algorithmic implementation and analysis. 


The SNN network is represented by a directed weighted graph $G=(V,A,W)$, with a special set of neurons $X \subset V$ called \emph{inputs} that have no incoming edges, and a subset of \emph{output} neurons\footnote{In contrast to the definition of \emph{circuits}, we do allow output neurons to have outgoing edges and self loops. The requirement will be that the value of the output neurons converges over time to the desired solution.} $Y \subset V$.  The neurons in the network can be either deterministic threshold gates or probabilistic threshold gates. As observed in biological networks, and departing from many artificial network models, neurons are either strictly inhibitory (all outgoing edge weights are negative) or excitatory (all outgoing edge weights are positive). The network evolves in discrete, synchronous \emph{rounds} as a Markov chain, where the firing probability of every neuron in round $\tau$ depends on the firing status of its neighbors in the preceding round $\tau-1$. For probabilistic threshold gates this firing is modeled using a standard sigmoid function. Observe that an SNN network is in fact, a \emph{distributed network}, every neuron responds to the firing spikes of its \emph{neighbors}, while having no global information on the entire network. 

\textbf{Remark.} In the setting of SNN, unlike classical distributed algorithms (e.g., \local\ or \congest), the algorithm is fully specified by the \emph{structure} of the network. That is, for a given network, its dynamic is fully determined by the  model. Hence, the key complexity measure here is the size of the network measured by the number of auxiliary neurons\footnote{I.e., neurons that are not the input or the output neurons.}. For certain problems, we also care for the tradeoff between the size and the computation time.
%
%

\vspace{-10pt}
\subsection{Measuring Time with Spiking Neural Networks} \vspace{-5pt}
We consider the algorithmic challenges of measuring time using networks of threshold gates and probabilistic threshold gates. We introduce the \emph{neural timer} problem defined as follows:  \vspace{-3pt}
\begin{mdframed}[hidealllines=false]
Given an input neuron $x$, an output neuron $y$, and a time parameter $t$, it is required to design a small neural network such that any firing of $x$ in a given round invokes the firing of $y$ for exactly the next $t$ rounds. 
\end{mdframed}\vspace{-5pt}
In other words, it is required to design a succinct timer, activated by the firing of its input neuron, that alerts when exactly $t$ rounds have passed.
%
%

A trivial solution with $t$ auxiliary neurons can be obtained by taking a directed chain of length $t$ (Fig. \ref{fig:det-counters}): the head of the chain has an incoming edge from the input $x$, the output $y$ has incoming edges from the input $x$, and all the other $t$ neurons on the chain. All these neurons are simple $OR$-gates, they fire in round $\tau$ if at least one of their incoming neighbors fired in round $\tau-1$.  Starting with the firing of $x$ in round $0$, in each round $i$, exactly one neuron, namely the $i^{th}$ neuron on the chain fires, which makes $y$ keep on firing for exactly $t$ rounds until the chain fades out. In this basic solution, the network spends one neuron that counts $+1$ and dies. It is noteworthy that the neurons in our model are very simple, they do not have any memory, and thus cannot keep track of the firing history. They can only base their firing decisions on the firing of their neighbors in the \emph{previous} round. 

With such a minimal model of computation, it is therefore intriguing to ask how to beat this linear dependency (of network size) in the time parameter $t$. Can we count to ten using only two (memory-less) neurons?
We answer this question in the affirmative, and show that even with just simple deterministic threshold gates, we can measure time up to $t$ rounds using only $O(\log t)$ neurons. It is easy to see that this bound is tight when using deterministic neurons (even when allowing some approximation). The reason is that $o(\log t)$ neurons encode strictly less than $t$ distinct configurations, thus in a sequence of $t$ rounds, there must be a configuration that re-occurs, hence locking the system into a state in which $y$ fires forever.
\vspace{-5pt}
\begin{theorem}[Deterministic Timers]\label{lem:upper-bound-det}
For every input time parameter $t\in \mathbb{N}_{>0}$, (1) there exists a deterministic neural timer network $\mathcal{N}$ with $O(\log t)$ deterministic threshold gates, (2) any deterministic neural timer requires $\Omega(\log t)$ neurons.
\end{theorem}\vspace{-5pt}
This timer can be easily adapted to the related problem of \emph{counting}, where the network should output the number of spikes (by the input $x$) within a time window of $t$ rounds.

\paragraph{Does Randomness Help in Time Estimation?}
Neural computation in general, and neural spike responses in particular, are inherently stochastic~\cite{lindner2009some}. One of our broader scope agenda is to understand the power and limitations of randomness in neural networks. Does neural computation become \emph{easier} or \emph{harder} due to the stochastic behavior of the neurons? 
%

We define a randomized version of the neural timer problem that allows some slackness both in the approximation of the time, as well as allowing a small error probability. For a given error probability $\delta \in (0,1)$, the output $y$ should fire for at least $t$ rounds, and must stop firing after at most $2t$ rounds\footnote{Taking $2t$ is arbitrary here, and any other constant greater than one would work as well.} with probability at least $1-\delta$.  
It turns out that this randomized variant leads to a considerably improved solution for $\delta=2^{-O(t)}$:
\vspace{-5pt}
\begin{theorem}[Upper Bound for Randomized Timers]\label{lem:upper-bound-rand}
For every time parameter $t \in \mathbb{N}_{>0}$, and error probability $\delta \in (0,1)$,  there exists a probabilistic neural timer network $\mathcal{N}$ with  $O(\min\{\log\log 1/\delta, \log t\})$ deterministic threshold gates plus \emph{additional} random spiking neuron. 
\end{theorem}
\vspace{-5pt}
Our starting point is a simple network with $O(\log 1/\delta)$ neurons, each firing independently with probability $1-1/t$. The key observation for improving the size bound into $O(\log\log 1/\delta)$ is to use the \emph{time axis}: we will use a \emph{single} neuron to generate random samples over time, rather than having \emph{many} random neurons generating these samples in a \emph{single} round. The deterministic neural counter network with time parameter of $O(\log 1/\delta)$ is used as a building block in order to gather the firing statistics of a single spiking neuron.
In light of the $\Omega(\log t)$ lower bound for deterministic networks, we get the first separation between deterministic and randomized solutions for error probability $\delta=\omega(1/2^t)$. This shows that randomness can help, but up to a limit: Once the allowed error probability is exponentially small in $t$, the deterministic solution is the best possible.
Perhaps surprisingly, we show that this behavior is tight:
\vspace{-5pt}
\begin{theorem}[Lower Bound for Randomized Timers]\label{lem:lower-bound-rand}
Any SNN network for the neural timer problem with time parameter $t$, and error $\delta \in (0,1)$ must use $\Omega(\min\{\log\log 1/\delta, \log t\})$ neurons. 
\end{theorem}
\vspace{-10pt}
\paragraph{Neural Counters.} 
Spiking neurons are believed to encode information via their firing rates. This underlies the \emph{rate coding} scheme \cite{adrian1926impulses,tsodyks1997neural,gerstner1997neural} in which the spike-count of the neuron in a given span of time is interpreted as a \emph{letter} in a larger alphabet. In a network of memory-less spiking neurons, it is not so clear how to implement this rate dependent behavior. How can a neuron convey a complicated message over time if its neighboring neurons remember only its recent spike? 
This challenge is formalized by the following neural counter problem: Given an input neuron $x$, a time parameter $t$, and $\Theta(\log t)$ output neurons represented by a vector $\bar{y}$, it is required to design a neural network such that the output vector $\bar{y}$ holds the binary representation of the number of times that $x$ fired in a sequence of $t$ rounds. As we already mentioned this problem is very much related to the neural timer problem and can be solved using $O(\log t)$ neurons. Can we do better?

The problem of maintaining a \emph{counter} using a small amount of space has received a lot of attention in the \emph{dynamic streaming} community. The well-known Morris algorithm \cite{Morris78,Flajolet85} maintains an approximate counter 
for $t$ counts using only $\log \log t$ bits. The high-level idea of this algorithm is to increase the counter with probability of $1/2^{C'}$ where $C'$ is the current read of the counter. The counter then holds the exponent of the number of counts. By following ideas of \cite{Flajolet85}, carefully adapted to the neural setting, we show:\vspace{-5pt}
\begin{theorem}[Approximate Counting]\label{thm:approx-counting}
For every time parameter $t$, and $\delta \in (0,1)$, there exists a randomized construction of approximate counting network using $O(\log\log t+\log (1/\delta))$ deterministic threshold gates plus an additional \emph{single} random spiking neuron, that computes an $O(1)$ (multiplicative) approximation for the number of input spikes in $t$ rounds with probability $1-\delta$.
\end{theorem}\vspace{-5pt}
We note that unlike the deterministic construction of timers that could be easily adopted to the problem of neural counting, our optimized randomized timers with $O(\log\log 1/\delta)$ neurons cannot be adopted into an approximate counter network. We therefore solve the latter by adopting Morris algorithm to the neural setting.
\vspace{-3pt}

\paragraph{Broader Scope: Lessons From Dynamic Streaming Algorithms.}
We believe that approximate counting problem provides just one indication for the potential relation between succinct neural networks and dynamic streaming algorithms.  In both settings, the goal is to 
gather statistics (e.g., over time) using a small amount of space. In the setting of neural network there are additional difficulties that do not show up in the streaming setting. E.g., it is also required to obtain fast \emph{update time}, as illustrated in our solution to the approximate counting problem.
\vspace{-10pt}
\subsection{Neural Synchronizers}\vspace{-5pt}
The standard model of spiking neural networks assumes that all edges (synapses) in the network have a uniform response latency. That is, the electrical signal is passed from the presynaptic neuron to the postsynaptic neuron within a fixed time unit which we call a \emph{round}. 
However, in real biological networks, the response latency of synapses can vary considerably depending on the biological properties of the synapse, as well as on the distance between the neighboring neurons. This results in an asynchronous setting in which different edges have distinct response time. 
We formalize a simple model of spiking neurons in the asynchronous setting, in which the given neural network also specifies a \emph{response latency} function $\ell: A \to \mathbb{R}_{\geq 1}$ that determines the number of rounds it takes for the signal to propagate over the edge.  Inspired by the synchronizers of Awerbuch and Peleg~\cite{AwerbuchP90b}, and using the above mentioned compressed timer and counter modules, we present a general simulation methodology (a.k.a synchronizers) that takes a network $\cN_{\sync}$ that solves the problem in the synchronized setting, and transform it into an ``analogous" network $\cN_{\async}$ that solves the same problem in the asynchronous setting.

 The basic building blocks of this transformation is the neural time component adapted to the asynchronous setting. 
The cost of the transformation is measured by the overhead in the number of neurons and in the computation time. Using our neural timers leads to a small overhead in the number of neurons.
%
%
%
\vspace{-5pt}
\begin{theorem}[Synchronizer, Informal]
There exists a synchronizer that given a network $\mathcal{N}_{\sync}$ with $n$ neurons and maximum response latency\footnote{I.e., $L$ correspond to the \emph{length} of the longest round.} $L$, constructs a network $\mathcal{N}_{\async}$ that has an ``analogous" execution in the asynchronous setting with a total number of $O(n + L\log L)$ neurons and a time overhead of $O(L^3)$.
\end{theorem}	\vspace{-5pt}
We note that although the construction is inspired by the work of Awerbuch and Peleg~\cite{AwerbuchP90b}, due to the large differences between these models, the precise formulation and implementation of our synchronizers are quite different. The most notable difference between the distributed and neural setting is the issue of memory: in the distributed setting, nodes can aggregate the incoming messages and respond when all required messages have arrived. In strike contrast, our neurons can only respond (by either firing or not firing) to signals arrived in the \emph{previous} round, and all signals from previous rounds cannot be locally stored.
For this reason and unlike \cite{AwerbuchP90b}, we must assume a bound on the largest edge latency.
In particular, in App.~\ref{sec:append-intro} we show that the size overhead of the transformed network $\mathcal{N}_{\async}$ must depend, at least logarithmically, on the value of the largest latency $L$. 
\vspace{-5pt}
\begin{Observation} \label{obs:lower-async}
The size overhead of any synchronization scheme is $\Omega(\log L)$.
\end{Observation}\vspace{-5pt}

This provably illustrates the difference in the overhead of synchronization between general distributed networks and neural networks. We leave the problem of tightening this lower bound (or upper bound) as an interesting open problem. 

\paragraph{Additional Related Work}
To the best of our knowledge, there are two main previous theoretical work on asynchronous neural networks.
Maass \cite{Maass94ECCC} considered a quite elaborated model for deterministic neural networks with \emph{arbitrary} response \emph{functions} for the edges, along with latencies that can be chosen by the network designer. Within this generalized framework, he presented a coarse description of a synchronization scheme that consists of various time modules (e.g., initiation and delay modules). Our work complements the scheme of \cite{Maass94ECCC} in the simplified SNN model by providing a rigorous implementation and analysis for size and time overhead. 
Khun et al. \cite{kuhn2010synchrony} analyzed the synchronous and asynchronous behavior under the stochastic neural network model of DeVille and Peskin \cite{deville2008synchrony}. Their model and framework is quite different from ours, and does not aim at building synchronizers. 


Turning to the setting of logical circuits, there is a long line of work on the asynchronous setting under various model assumptions \cite{armstrong1969design,hauck1995asynchronous, sparso2001asynchronous,bjerregaard2006survey,manohar2017eventual} that do not quite fit the memory-less setting of spiking neurons.

\paragraph{Comparison with Concurrent Work \cite{lynch2019counter}.} Independently to our work, Wang and Lynch proposed a similar construction for the neural counter problem. Their work restricts attention to deterministic threshold gates and do not consider the neural timer problem and synchronizers which constitute the main contribution of our paper. We note that our approximate counter solution with $O(\log\log t+\log (1/\delta))$ neurons resolves the open problem stated in \cite{lynch2019counter}.
%
%
%
%

\vspace{-10pt}\subsection{Preliminaries} \label{sec:perlim} \vspace{-5pt}
We start  by defining our model along with useful notation.

\textbf{A Neuron.} A \emph{deterministic} neuron $u$ is modeled by a \emph{deterministic} threshold gate. 
Letting $b(u)$ to be the threshold value of $u$. Then it outputs $1$ if the weighted sum of its incoming neighbors exceeds $b(u)$. A \emph{spiking neuron} is modeled by a probabilistic threshold gate that fires with a sigmoidal probability $p(x)= \frac{1}{1+e^{-x}}$ where $x$ is the difference between the weighted incoming sum of $u$ and its threshold $b(u)$. 

\textbf{Neural Network Definition.}
A \emph{Neural Network} (NN) $\cN =\langle X,Z,Y, w,b \rangle$ consists of $n$ input neurons $X=\{x_1, \ldots, x_{n}\}$, $m$ output neurons $Y=\{y_1, \ldots, y_{m}\}$, and $\ell$ auxiliary neurons $Z = \{z_1,...,z_{\ell} \}$. 
In a \emph{deterministic} neural network (DNN) all neurons are deterministic threshold gates. In spiking neural network (SNN), the neurons can be either deterministic threshold gates or probabilistic threshold gates. 
The directed weighted synaptic connections between $V=X \cup Z \cup Y$ are described by the weight function $w: V \times V \rightarrow \mathbb{R}$. A weight $w(u,v) =0$ indicates that a connection is not present between neurons $u$ and $v$. Finally, for any neuron $v$, $\Bias(v) \in \mathbb{R}_{\geq 0}$ is the threshold value (activation bias).
%
The weight function defining the synapses is restricted in two ways.
The in-degree of every input neuron $x_i$ is zero, i.e., $w(u,x) = 0$ for all $u \in V$ and $x \in X$. 
Additionally, each neuron is either inhibitory or excitatory: 
if $v$ is inhibitory, then $w(v,u)\leq 0$ for every $u$, and if $v$ is 
excitatory, then $w(v,u)\geq 0$ for every $u$. 

\textbf{Network Dynamics.}
The network evolves in discrete, synchronous rounds as a Markov chain. 
The firing probability of every neuron in round $\tau$ depends on the firing status of its neighbors in round $\tau-1$, via a standard sigmoid 
function, with details given below.
For each neuron $u$, and each round $\tau \ge 0$, let $u^{\tau}=1$ if $u$ fires (i.e., generates a spike) in round $\tau$. Let $u^{0}$ denote the initial firing state of the neuron. The firing state of each input neuron $x_j$ in each round is the input to the network. For each non-input neuron $u$ and every round $\tau \ge 1$, let $\pot(u,\tau)$ denote the membrane potential at round $\tau$ and $p(u,\tau)$ denote
the firing probability ($\Pr[u^\tau = 1]$), calculated as:
\begin{align}
\label{eq:potentialOut}
\pot(u,\tau)= \hspace{-.65em}  \sum_{v \in V}w_{v,u}\cdot v^{\tau-1} -b(u) 
\text{ and }p(u,\tau)=\frac{1}{1+e^{-\frac{\pot(u,\tau)}{\lambda}}}
\end{align}
where $\lambda > 0$ is a \emph{temperature parameter} which determines the steepness of the sigmoid. Clearly, $\lambda$ does not affect the computational power of the network (due to scaling of edge weights and thresholds), thus we set $\lambda=1$. In \emph{deterministic} neural networks (DNN), each neuron $u$ is a deterministic threshold gate that fires in round $\tau$ iff $\pot(u,\tau)\geq 0$. 

\textbf{Network States (Configurations).}
Given a network $\mathcal{N}$ (either a DNN or SNN) with $N$ neurons, the configuration (or \emph{state}) of the network in time $\tau$ denoted as $s_\tau$ can be described as an $N$-length binary vector indicating which neuron fired in round $\tau$. 

\textbf{The Memoryless Property.} 
The neural networks have a memoryless property, in the sense that each state depends only on the state of the previous round. In a DNN network, the state $s_{\tau-1}$ fully determines $s_\tau$. 
In an SNN network, for every fixed state $s^*$ it holds
$\Pr[s_\tau=s^* \ | \ s_1,...s_{\tau-1}] = \Pr[s_\tau=s^* \ | \ s_{\tau-1}]$.
Moreover for any $\tau, \tau',r >0$, it holds that $\Pr[s_{\tau+r}=s^* \ | \  s_{\tau}] = \Pr[s_{\tau'+r}=s^* \ | \ s_{\tau'}]$.
%

%
%
%

\textbf{Hard-Wired Inputs.} We consider neural networks that \emph{solve} a given parametrized problem (e.g., neural timer with time parameter $t$).
The parameter to the problem can be either hard-wired in the network or alternatively be given as part of the input layer to the network. In most of our constructions, the time parameter is hard-wired. In some cases, we also show constructions with soft-wiring.


\vspace{-5pt}
\begin{figure}[h]
	\includegraphics[scale=0.26]{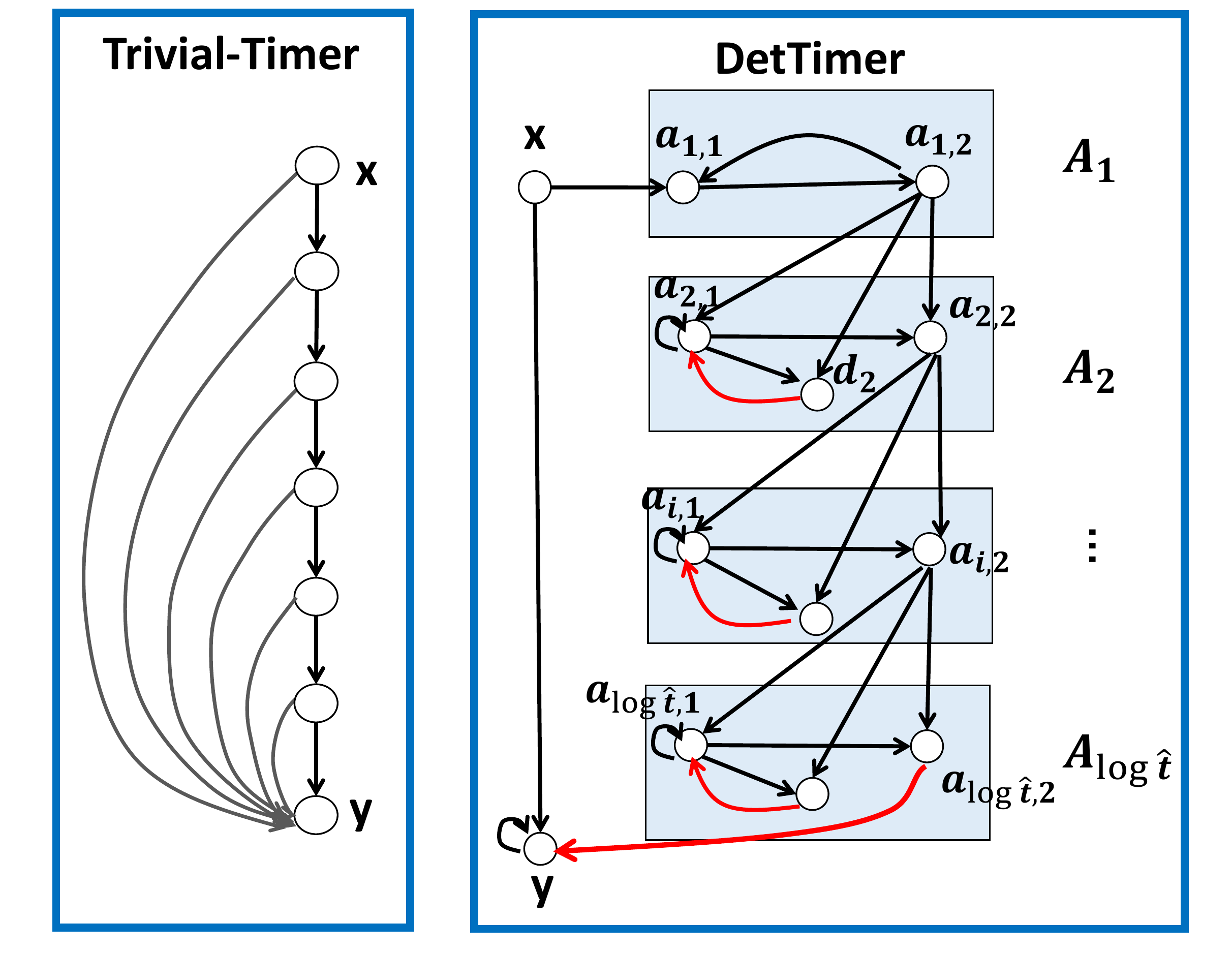}
	\centering
	\includegraphics[scale=0.28]{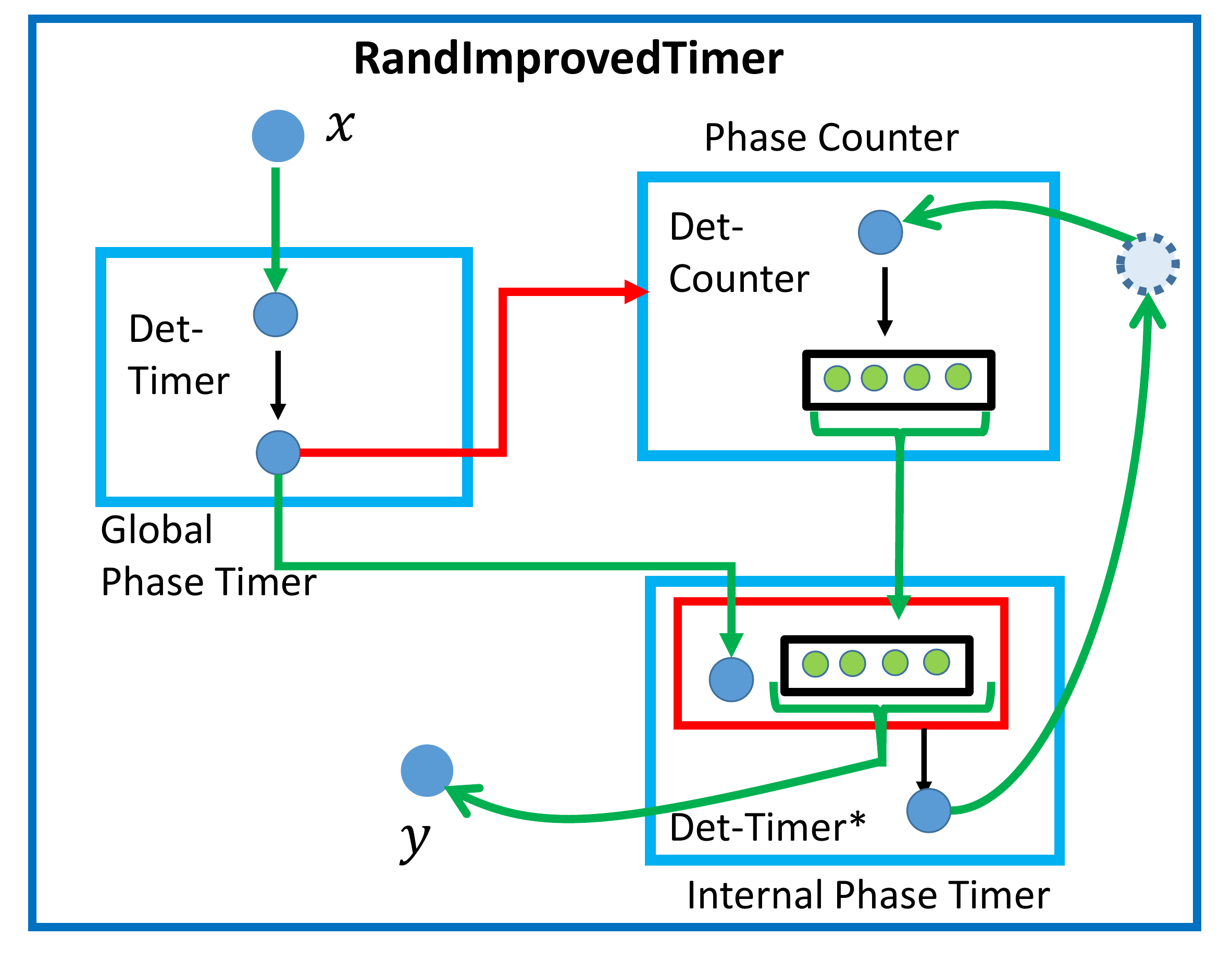}
	\caption{\small{Illustration of timer networks with time parameter $t$. Left: The na\"{\i}ve  timer with $\Theta(t)$ neurons. Mid: deterministic timer with $\Theta(\log t)$ neurons. Right: randomized timer with $O(\log\log 1/\delta))$ neurons, using the $\DetTimer$ modules with parameter $t'=\log 1/\delta$.}\label{fig:det-counters}}
\end{figure}
\vspace{-25pt}
\section{Deterministic Constructions of Neural Timer Networks}\label{sec:det-timer}\vspace{-10pt}
As a warm-up, we start by considering deterministic neural timers.
\begin{definition}[Det. Neural Timer Network]
Given time parameter $t$, a \emph{deterministic} neural timer network $\mathcal{DT}$ is a network of threshold gates, with an input neuron $x$, an output neuron $y$, and additional auxiliary neurons. The network satisfies that in every round $\tau$, $y^\tau=1$ iff there exists a round $\tau>\tau'\geq \tau - t$ such that $x^{\tau'}=1$. 
\end{definition}
%
%
\vspace{-5pt}
\paragraph{Lower Bound (Pf. of Thm. \ref{lem:upper-bound-det}(2)).}
For a given neural timer network $\cN$ with $N$ auxiliary neurons, recall that the \emph{state} of the network in round $\tau$ is described by an $N$-length vector indicating the firing neurons in that round.
Assume towards contradiction that there exists a neural timer with $N \leq \log t-1$ auxiliary neurons. 
Since there are at most $2^N$ different states, by the pigeonhole principle, there must be at least two rounds $\tau, \tau' \leq t-1$ in which the state of the network is identical, i.e., where $s_{\tau}=s_{\tau'}=s^*$ for some $s^* \in \{0,1\}^N$. By the correctness of the network, the output neuron $y$ fires in all rounds $\tau'' \in [\tau+1,\tau'+1]$.  By the memoryless property, we get that $s_{\tau''}=s^*$ for $\tau''=\tau+i \cdot (\tau'-\tau)$ for every $i \in \mathbb{N}_{\geq 0}$. Thus $y$ continues firing forever, in contradiction that it stops firing after $t$ rounds.  Note that this lower bound holds even if $y$ is allowed to stop firing in any finite time window.

\paragraph{A Matching Upper Bound (Pf. Thm. \ref{lem:upper-bound-det}(1)).}
For ease of explanation, we will sketch here the description of the network assuming that it is applied only once (i.e., the input $x$ fires once within a window of $t$ rounds).
Taking care of the general case requires slight adaptations\footnote{I.e., whenever $x$ fires again in a window of $t$ rounds, one should reset the timer and start counting $t$ rounds from that point on.}, see Appendix~\ref{sec:append-det} for the complete details.

At the high-level, the network consists of $k=\Theta(\log t)$ layers $A_1, \ldots , A_k $ each containing two excitatory neurons $a_{i,1}, a_{i,2}$ denoted as \emph{counting} neurons, and one inhibitory neuron $d_i$. Each layer $A_i$ gets its input from layer $A_{i-1}$ for every $i\geq 2$, and $A_1$ gets its input from $x$. The role of each layer $A_i$ is to count \emph{two} firing events of the neuron $a_{i-1,2} \in A_{i-1}$. Thus the neuron $a_{\log t,2}$ counts $2^{\log t}$ rounds. 

Because our network has an update time of $\log t$ rounds (i.e., number of rounds to update the timer), for a given time parameter $t$, the construction is based on the parameter $\hat{t}$ where $\hat{t}+\log \hat{t} = t$.
\vspace{-10pt}
 \begin{itemize}
 	\item The first layer $A_1$ consists of two neurons $a_{1,1}, a_{1,2}$. The first neuron $a_{1,1}$ has positive incoming edges from $x$ and $ a_{1,2}$ with weights $w(x,a_{1,1})=3$ , $w(a_{1,2}, a_{1,1})=1$, and threshold $b(a_{1,1})=1$. The second neuron $a_{1,2}$ has an incoming edge from $a_{1,1}$ with weight $w(a_{1,1}, a_{1,2})=1$ and threshold $b(a_{1,2}) = 1$.
 	Because we have a loop going from $a_{1,1}$ to $a_{1,2}$ and back, once $x$ fired  $a_{1,2}$ will fire every two rounds.
 	
\item For every $i=2 \ldots \log \hat{t}$, the $i^{th}$ layer $A_i$ contains $3$ neurons, two \emph{counting} neurons $a_{i,1}$, $a_{i,2}$ 
and a reset neuron $d_i$. 
The first neuron $a_{i,1}$ has positive incoming edges from $a_{i-1,2}$, and a self loop with weight $w(a_{i-1,2},a_{i,1})=w(a_{i,1},a_{i,1})=1$, a negative incoming edge from $d_i$ with weight $w(d_i, a_{i,1})=-1$, and threshold $b(a_{i,1})=1$. The second counting neuron $a_{i,2}$ has incoming edges from $a_{i-1,2}$ and $a_{i,1}$ with weight $w(a_{i-1,2},a_{i,2})=w(a_{i,1},a_{i,2})=1$, and threshold $b(a_{i,2})=2$. The reset neuron $d_i$ is an inhibitor copy of $a_{i-1,2}$ and therefore also has incoming edges from $a_{i-1,2}$ and $a_{i,1}$ with weight $w(a_{i-1,2},d_i)=w(a_{i,1},d_i)=1$ and threshold $b(d_i)=2$. As a result, $a_{i,1}$ starts firing after $a_{i-1,2}$ fires once, and $a_{i,2}$ fires after $a_{i-1,2}$ fires twice. Then the neuron $d_i$ inhibits $a_{i,1}$ and the layer is ready for a new count. 
 	
\item The output neuron $y$ has a positive incoming edge from $x$ as well as a self-loop with weights $w(x,y)=2$, $w(y,y)=1$. In addition, it has a negative incoming edge from the last counting neuron $a_{\log \hat{t},2}$ with weight $w(a_{\log \hat{t},2},y)=-1$ and threshold $b(y)=1$. Hence, after $x$ fires the output $y$ continues to fire as long as $a_{\log \hat{t},2}$ did not fire.
 	
\item The last counting neuron  $a_{\log \hat{t},2}$ also has negative outgoing edges to all counting neurons (neurons of the form $a_{i,j}$) with weight $w(a_{\log \hat{t},2}, a_{i,j})=-2$. As a result, after the timer counts $t$ rounds it is reset.  	
 \end{itemize}
The key claim that underlines the correctness of Thm.~\ref{lem:upper-bound-det}(1) is as follows. 
\begin{claim} \label{clm:layers}
	If $x$ fires in round $t_0$, for each layer $i$ the neuron $a_{i,2}$ fires in rounds $t_0 + \ell \cdot 2^i+i-1$ for every $\ell = 1 \ldots \floor{\hat{t}/2^i}$.
\end{claim}
\begin{proof}
	The proof is by induction on $i$. For $i=1$, once $x$ fires in round $t_0$, neuron $a_{1,1}$ fires in round $t_0+1$ and $a_{1,2}$ fires in round $t_0+2$. Because there is a bidirectional edge between $a_{1,1}$ and $a_{1,2}$, the second counting neuron $a_{1,2}$ keeps firing every two rounds. Assume the claim holds for neuron $a_{i-1,2}$, and consider the $i^{th}$ layer $A_i$. Recall that $a_{i,2}$ fires in round $t'$ only if $a_{i,1}$ and $a_{i-1,2}$ fired in round $t'-1$. The neuron $a_{i,1}$ fires one round after $a_{i-1,2}$ fires and keeps firing as long as $d_i$ did not fire. By the induction assumption $a_{i-1,2}$ fired for the first time in round $2^{i-1}+i-2$ and therefore $a_{i,1}$ starts firing in round $2^{i-1}+i-1$. Note that in round $2^{i-1}+i-1$ the neuron $a_{i-1,2}$ did not fire, and therefore the neurons  $a_{i,2}$ and $d_i$ can start firing only after $a_{i-1,2}$ fires again. Hence, only in round $2 \cdot 2^{i-1}+i-2+1 = 2^{i}+i-1$ the neurons $a_{i,2}$ and $d_i$ fires for the first time. In the next round, because of the inhibition of $d_i$ both counting neurons $a_{i,1}$ and $a_{i,2}$ do not fire and we can repeat the same arguments considering the next time the counting neurons $a_{i,1}$, $a_{i,2}$ fire. 
	
	We note that once the neuron  $a_{\log \hat{t},2}$ fires for the first time in round $t_0 + 2^{\log \hat{t}}+\log \hat{t}-1 =  t_0 + \hat{t}+\log \hat{t}-1 $, it inhibits all the counting neurons. Hence, as long as $x$ did not fire again, all counting neurons will be idle starting at round $t_0 + \hat{t}+\log \hat{t} = t_0 + t$.
\end{proof}
The complete proof of Thm.~\ref{lem:upper-bound-det}(1) is given in Appendix~\ref{sec:det-proof}. \\
\textbf{Timer with Time Parameter.}
In Appendix~\ref{sec:mod-timer}, we show a slight modified variant of neural timer denoted by $\DetTimer^*$ which also receives as input an additional set of $\log t$ neurons that encode the desired duration of the timer. This modified variant is used in our improved randomized constructions.

\vspace{-5pt}
\paragraph{Neural Counters.}
In Appendix \ref{sec:mod-timer} we show a modification of the timer into a counter network $\DetCounter$ that instead of counting the number of rounds, counts the number of input spikes in a time interval of $t$ rounds. 
\vspace{-3pt} 
\begin{lemma} \label{lem:det-counter}
	Given time parameter $t$, there exists a deterministic \emph{neural counter} network which has an input neuron $x$, a collection of $\log t$ output neurons represented by a vector $\bar{y}$, and $O(\log t)$ additional auxiliary neurons. In a time window of $t$ rounds, for every round $\tau$, if $x$ fired $r_\tau$ times in the last $\tau$ rounds, the output $\bar{y}$ encodes $r_{\tau}$ by round $\tau+\log r_\tau+1$. 
\end{lemma} \vspace{-3pt} 
This extra-additive factor of $\log r_\tau$ is due to the update time of the counter.
In Appendix \ref{sec:counting}, we revisit the neural counter problem and provide an \emph{approximate} randomized solution with $O(\log\log t+\log(1/\delta))$ many neurons where $\delta$ is the error parameter. This construction is based on the well-known Morris algorithm (using the analysis of \cite{Flajolet85}) for approximate counting in the streaming model.
%

\vspace{-7pt}
\section{Randomized Constructions of Neural Timer Networks}\label{sec:rand-timer}\vspace{-5pt}
We now turn to consider randomized implementations.
The input to the construction is a time parameter $t$ and an error probability $\delta \in (0,1)$, that are hard-wired into the network. 

\begin{definition}[Rand. Neural Timer Network]
A randomized neural timer $\mathcal{RT}$ for parameters $t \in \mathbb{N}_{>0}$ and $\delta \in (0,1)$, satisfies the following for a time window of $\poly(t)$ rounds.
\begin{itemize}
\item For every fixed firing event of $x$ in round $\tau$, with probability $1-\delta$, $y$ fires in each of the following $t$ rounds.
\item $y^{\tau'}=0$ for every round $\tau'$ such that $\tau'- \Last(\tau') \geq 2t$ with probability $1- \delta$, where $\Last(\tau')=\max\{i \leq \tau' ~\mid~ x^i=1\}$ is the last round $\tau$ in which $x$ fired up to round $\tau'$.
\end{itemize}
\end{definition}

Note that in our definition, we have a success guarantee of $1-\delta$ for any fixed firing event of $x$, on the event that $y$ fires for $t$ many rounds after this firing. In contrast, with probability of $1-\delta$ over the entire span of $\poly(t)$ rounds, $y$ \emph{does not} fire in cases where the last firing of $x$ was $2t$ rounds apart. 
We start by showing a simple construction with $O(\log 1/\delta)$ neurons. 

\subsection{Warm Up: Randomized Timer with $O(\log 1/\delta)$ Neurons}
The network $\RandBasic(t,\delta)$ contains a collection of $\ell=\Theta(\log 1/\delta)$ spiking neurons $A=\{a_1,\ldots, a_\ell\}$ that can be viewed as a \emph{time-estimator population}. Each of these neurons have a positive self loop, a positive incoming edge from the input neuron $x$, and a positive outgoing edge to the output neuron $y$. See Figure~\ref{fig:random} for an illustration. Whereas these $a_i$ neurons are probabilistic spiking neurons\footnote{A neuron that fires with a probability specified in Eq. (\ref{eq:potentialOut})}, the output $y$ is simply a threshold gate. We next explain the underlying intuition. 
Assume that the input $x$ fired in round $0$. It is then required for the output neuron $y$ to fire for at least $t$ rounds $1,\ldots, t$, and stop firing after at most $2t$ rounds with probability $1-\delta$. By having every neuron $a_i$ fires (independently) w.p $(1-1/t)$ in each round given that it fired in the previous round\footnote{A neuron $a_i$ that stops firing in a given round, drops out and would not fire again with good probability.}, we get that $a_i$ fires for $t$ \emph{consecutive} rounds with probability $(1-1/t)^t\approx 1/e$. On the other hand, it fires for $2t$ consecutive rounds with probability $(1-1/t)^{2t}=1/e^2$. Since we have $\Theta(\log 1/\delta)$ many neurons, by a simple application of Chernoff bound, the output neuron $y$ (which simply counts the number of firing neurons in $A$) can distinguish between round $t$ and round $2t$ with probability $1-\delta$.

\begin{figure}[h] 
\begin{center}
\includegraphics[scale=0.3]{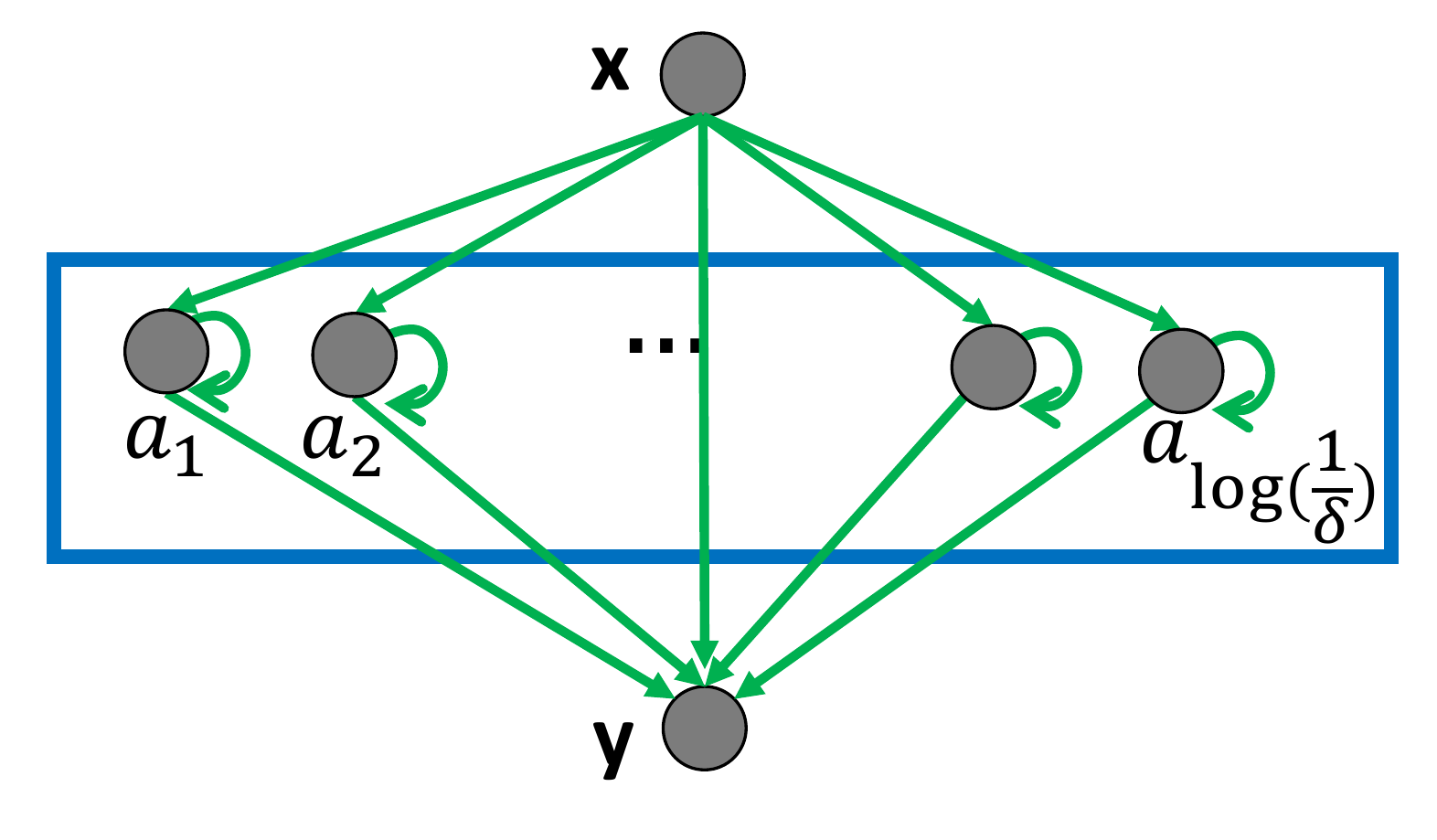}
\end{center}
\caption{\small{Illustration of the $\RandBasic(t,\delta)$ network. Each neuron $a_i$ fires with probability $1-1/t$ in round $\tau$ given that it fired in the previous round, and therefore fires for $t$ consecutive rounds with constant probability. The output $y$ fires if at least $1/(2e)$ fraction of the $a_i$ neurons fired in the previous round.  \label{fig:random}}}
\end{figure}
\paragraph{Detailed Construction.}
The network $\RandBasic(t,\delta)$ has input neuron $x$, output neuron $y$, and $\ell=\Theta(\log 1/\delta)$ spiking neurons $A=\{a_1,\ldots, a_\ell\}$.
We set the weights of the self loop of each $a_i$, and the weight of the incoming edge from $x$ to be $w(x,a_i)=w(a_i,a_i)=\log(t-1)+b(a_i)$. The threshold value of $a_i$ is set to $b(a_i)=\Theta (\log(t\ell/\delta))$. This makes sure that given a firing of either $x$ or $a_i$ in round $\tau$, the probability that $a_i$ fires in round $\tau+1$ is $1-1/t$. In the complementary case (neither $x$ nor $a_i$ fired in round $\tau$), $a_i$ fires in round $\tau$ with probability at most $O(\delta/\poly(t\ell))$. 
For the output $y$, we set $w(a_i,y)=1$ for each $a_i$, the weight of the edge from $x$ to be $w(x,y)= \frac{\ell}{2e}$, and its threshold $b(y)=\frac{\ell}{2e}$. This makes sure that $y$ fires in round $\tau'$ if either $x$ or at least $1/2e$ fraction of the $a_i$ neurons fired in round $\tau'-1$. 
We next analyze the construction.
\begin{lemma}[Correctness]
Within a time window of $\poly(t)$ rounds it holds that:
\begin{itemize}
\item For every fixed firing event of $x$ in round $\tau$, with probability $1-\delta$, $y$ fires in each of the following $t$ rounds.
\item $y^{\tau'}=0$ for every round $\tau'$ such that $\tau'- \Last(\tau') \geq 2t$ with probability at least $1- \delta$.
\end{itemize}	
\end{lemma}
\begin{proof}
When $x$ fires in round $\tau_0$, each neuron $a_i$ fires for the following $t$ consecutive rounds independently with probability $1/e$. Therefore, the expected number of neurons in $A$ that fired for $t$ consecutive rounds starting round $\tau_0+1$ is $\frac{\ell}{e}$. Using Chernoff bound upon picking a large enough constant $c$ s.t $\ell= c \cdot \log(1/\delta)$, at least $\ell/2e$ auxiliary neurons fired for $t$ consecutive rounds and $y$ fires in rounds $[\tau_0+2,\tau_0+t]$ with probability $1 - \delta$. Since $y$ has an incoming edge from $x$, it fires in round $\tau_0+1$ as well.

Next, recall that for each neuron $a_i\in A$, given that $a_i$ or $x$ did not fire in round $\tau$, the probability that $a_i$ fires in round $\tau+1$ is at most $\delta/\poly(\ell t)$. Hence by union bound, in a window of $\poly(t)$ rounds,  the probability there exists a neuron $a_i\in A$ that fired in round $\tau'$ but did not fire in round $\tau'-1$ is at most $\delta/2$. Assuming no $a_i \in A$ fires unless it fired previously, each $a_i\in A$ fires for $2t$ consecutive rounds with probability $1/e^2$. Using Chernoff bound the probability at least $\frac{\ell}{2e}$ neurons from $A$ fired for $2t$ consecutive rounds is at most $\delta/2$ (again we choose $\ell$ accordingly).
Thus, we conclude that the probability there exists a round $\tau'$ s.t $\tau'- \Last(\tau') \geq 2t$ in which $y^{\tau'}=1$ is at most $\delta$.
\end{proof}
\subsection{Improved Construction with $O(\log\log 1/\delta)$ Neurons}
We next describe an optimal randomized timer $\RandImprovedTimer$ with an exponentially improved number of auxiliary neurons. This construction also enjoys the fact that it requires a \emph{single} spiking neuron, while the remaining neurons can be deterministic threshold gates. 
Due to the tightness of Chernoff bound, one cannot really hope to estimate time with probability $1-\delta$ using $o(\log(1/\delta))$ samples. Our key idea here is to generate the same number of samples by re-sampling one particular neuron over several rounds.  
Intuitively, we are going to show that for our purposes having $\ell=\log(1/\delta)$ neurons $a_1, \ldots, a_\ell$ firing with probability $1-1/t$ in a \emph{given} round is \emph{equivalent} to having a \emph{single} neuron $a^*$ firing with probability $1-1/t$ (independently) in a sequence of $\ell$ rounds. 

Specifically, observe that the distinction between round $t$ and $2t$ in the $\RandBasic$ network is based only on the \emph{number} of spiking neurons in a given round. In addition, the distribution on the number of times $a^*$ fires in a span of $\ell$ rounds is equivalent to the distribution on the number of firing neurons $a_1, \ldots, a_\ell$ in a given round. For this reason, every phase of $\RandImprovedTimer$ simulates a single round of $\RandBasic$. To count the number of firing events in $\ell$ rounds, we use the deterministic neural counter module with $\log \ell=O(\log\log 1/\delta)$ neurons.

We now further formalize this intuition. The network $\RandImprovedTimer$ simulates each round of $\RandBasic$ using a phase of $\ell' =\Theta(\log 1/\delta)$ rounds \footnote{Due to tactical reasons each phase consists of $\ell'=\ell + \log \ell$ rounds instead of $\ell$.}, but with only $O(\log\log 1/\delta)$ neurons. In the $\RandBasic$ network each of the neurons $a_i$ fires (independently) in each round w.p $1-1/t$. Once it stops firing in a given round, it basically drops out and would not fire again with good probability. Formally, consider an execution of the $\RandBasic$ and let $n_i$ be the number of neurons in $A$ that fired in round $i$. In round $i+1$ of this execution, we have $n_i$ many neurons each firing w.p $1-1/t$ (while the remaining neurons in $A$ fire with a very small probability).
In the corresponding $i+1$ phase of the network $\RandImprovedTimer$, the chief neuron $a^*$ fires w.p $1-1/t'$ where $t'= \frac{t}{\ell'}$ for $n'_i \leq \ell $ consecutive rounds\footnote{Note that because each phase takes $\ell'=\Theta(\log 1/\delta)$ rounds, we will need to count $t'= \frac{t}{\ell'}$ many phases. Thus $a^*$ fires with probability $1-1/t'$ rather then w.p $1-1/t$.} where $n'_i$ is the number of rounds in which $a^*$ fired in phase $i$. 

The dynamics of the network $\RandImprovedTimer$ is based on discrete phases. Each phase has a fixed number of $\ell' = O(\ell)$ rounds, but has a possibly different number of \emph{active rounds}, namely, rounds in which $a^*$ attempts firing. Specifically, a phase $i$ has an active part of $n'_i$ rounds where $n'_i$ is the number of rounds in which $a^*$ fired in phase $i-1$. In the remaining $\ell'-n'_i$ rounds of that phase, $a^*$ is idle.  
To implement this behavior, the network should keep track of the number of rounds in which $a^*$ fires in each phase, and supply it as an input to the next phase (as it determines the length of the active part of that phase).
For that purpose we will use the deterministic modules of neural timers and counters. The module $\DetCounter$ with time parameter $\Theta(\log 1/\delta)$ is responsible for counting the number of rounds that $a^*$ fires in a given phase $i$. The output of this module at the end of the phase is the input to a $\DetTimer^*$ module\footnote{Here we use the variant of $\DetTimer$ in which the time is encoded in the input layer of the network.} in the beginning of phase $i+1$. In addition, we also need a \emph{phase timer} module $\DetTimer$ with time parameter $\Theta(\log 1/\delta)$ that ``announces" the end of a phase and the beginning of a new one. Similarly to the network $\RandBasic$, the output neuron $y$ fires as long as $a^*$ fires for at least $(1/2e)$ fraction of the rounds in each phase (in an analogous manner as in the $\RandBasic$ construction). 
See Fig. \ref{fig:loglog-random} for an illustration of the network.
Note that since we only use deterministic modules with time parameter $\Theta(\log 1/\delta)$, the total number of neurons (which are all threshold gates) will be bounded by $O(\log \log 1/\delta)$. 
We next give a detailed description of the network and prove Thm.~\ref{lem:upper-bound-rand}.

\paragraph{Complete Proof of Thm.~\ref{lem:upper-bound-rand}:}
We first describe the modules of the network $\RandImprovedTimer$ that gets as input the time parameter $t$ and error probability $\delta$.

\paragraph{Network Modules:}
\begin{itemize}
\item A \emph{Global-Phase-Timer} module implemented by a (slightly modified) module of $\DetTimer(\ell')$. Due to the update time of $\DetCounter$ (lemma~\ref{lem:det-counter}), we set the length of each phase to $\ell'= \ell+ \log \ell$ where $\ell$ correspond to the number of spiking neurons in $\RandBasic$. 
Upon initializing this timer, the output neuron of this module fires \emph{after} $\ell'$ rounds (instead of firing \emph{for} $\ell'$ rounds). This firing is the wake-up call for the network that a phase has terminated ($\ell'$ rounds have passed). This will activate some cleanup steps, and a subsequent ``announcement" for the start of a new phase. 

To allow this module to inhibit as well as excite other neurons in the network, we will have two output 
(copy) neurons, one will be inhibitor and the other excitatory. The inhibitor activates a clean-up round (in order to clear the counting information from the previous phase). After one round, using a delay neuron the excitatory neuron safely announces the beginning of a new phase.  	
\item An \emph{Internal-Phase-Timer} module also implemented by a (yet a differently slightly modified) variant of $\DetTimer$. The role of this module is to indicate to the spiking neuron $a^*$ the number of rounds in which it should attempt firing in each phase. Recall that each phase $i$ starts by an active part of length $n_i$ in which $a^*$ attempts firing w.p. $1-1/t'$ in each of these rounds. In the remaining $\ell'-n_i$ rounds till the end of the phase, $a^*$ is idle. In each phase $i$, we then set the internal timer to $n_i$, this will activate $a^*$ for $n_i$ rounds. The time parameter $n_i$ is given as input to this module. For that purpose, we use the $\DetTimer^*$ variant in which the time parameter is given as an input. In our case, this input is supplied by the output layer of the counting module (describe next) at the end of phase $i-1$. 
In particular, at the end of the phase, the output of the counting module is fed into the input layer of the \emph{Internal-Phase-Timer} module. Then, the information will be deleted from the counting module, ready to maintain the counting in the next phase.

Since we would need to keep on providing the counting information throughout the entire phase, we augment the input layer of this module by self loops that keeps on presenting this information thought the phase.
	
\item A \emph{Phase-Counter} module implemented by the $\DetCounter$ network, maintains the number of rounds in which $a^*$ fires in the current phase. At the end of every $i^{th}$ phase, the output layer of this module stores the number of rounds in which $a^*$ fired in phase $i$.  
At the end of the phase, upon receiving a signal from the \emph{Global-Phase-Timer}, the output layer copies its information to the input layer of the \emph{Internal-Phase-Timer} module using an intermediate layer of neurons, and the information of the module is deleted (by inhibitory connections from the \emph{Global-Phase-Timer} module).

\paragraph{Complete Description (Edge Weights, Bais Values, etc.)}
\item The neuron $a^*$ has threshold $b(a^*) = \Theta(\log (\ell t /\delta))$, and a positive incoming edge from the output $z_1$ of the \emph{Internal-Phase-Timer} module with weight $w(z_1,a^*)=\ln (t'-1) + b(a^*)$. Therefore $a^*$ fires with probability $1-1/t'$ if $z_1$ fired in the previous round, and w.h.p\footnote{Here high probability refers to probability of $1-\delta/\poly(t)$.} does not fire otherwise.

\item Each neuron in the output of the \emph{Phase-Counter} has a positive outgoing edge to an intermediate \emph{copy} neuron $c_i$ with weight $1$. In addition, each $c_i$ has a positive incoming edge from the \emph{Global-Phase-Timer} excitatory output with weight $1$, and threshold $b(c_i)=2$. The copy neurons have outgoing edges to the input of the \emph{Internal-Phase-Timer} and are used to copy the current count for the next phase.

\item The inhibitor output of the \emph{Global-Phase-Timer} has outgoing edges to all neurons in the \emph{Internal-Phase-Timer} and \emph{Phase-Counter} with weight $-5$. This is used to clean-up the out-dated counting information at the end of the phase. 

\item The excitatory output of the \emph{Global-Phase-Timer} has an outgoing edge to a delay neuron $d$ with weight $1$ and threshold $b(d)=1$. Hence, $d$ fires one round after a phase ended, and alerts the beginning of the new phase. The neuron $d$ has outgoing edges to the input of \emph{Global-Phase-Timer} and \emph{Internal-Phase-Timer} with large weight.

\item The output neuron $y$ has incoming edges from the time input neurons $q_1, \ldots, q_{\log \ell}$ of the \emph{Internal-Phase-Timer} module each with weight $w(q_i,y)=1$ and threshold $b(y)=\frac{\ell}{2e}$.  Therefore $y$ fires if $a^*$ fired for at least $\frac{\ell}{2e}$ times in the previous phase. In addition, $y$ has positive incoming edges from $x$ and the delay neuron $d$ of the \emph{Global-Phase-Timer} module, each with weight $\frac{\ell}{2e}$. This insures that $y$ also fires between phases. 

\item The \emph{Global-Phase-Timer} input has an incoming edge from $x$ with large weight, in order to initialize the timer when the input $x$ fires. In addition, $x$ has outgoing edges with large weight to the time input of the \emph{Internal-Phase-Timer}, such that the decimal value of the input is set to $\ell$. 
\end{itemize}
All neurons except for $a^*$ are threshold gates, 
see Figure~\ref{fig:loglog-random} for a schematic description of the $\RandImprovedTimer$ network. 
\begin{figure} 
\begin{center}
\includegraphics[scale=0.4]{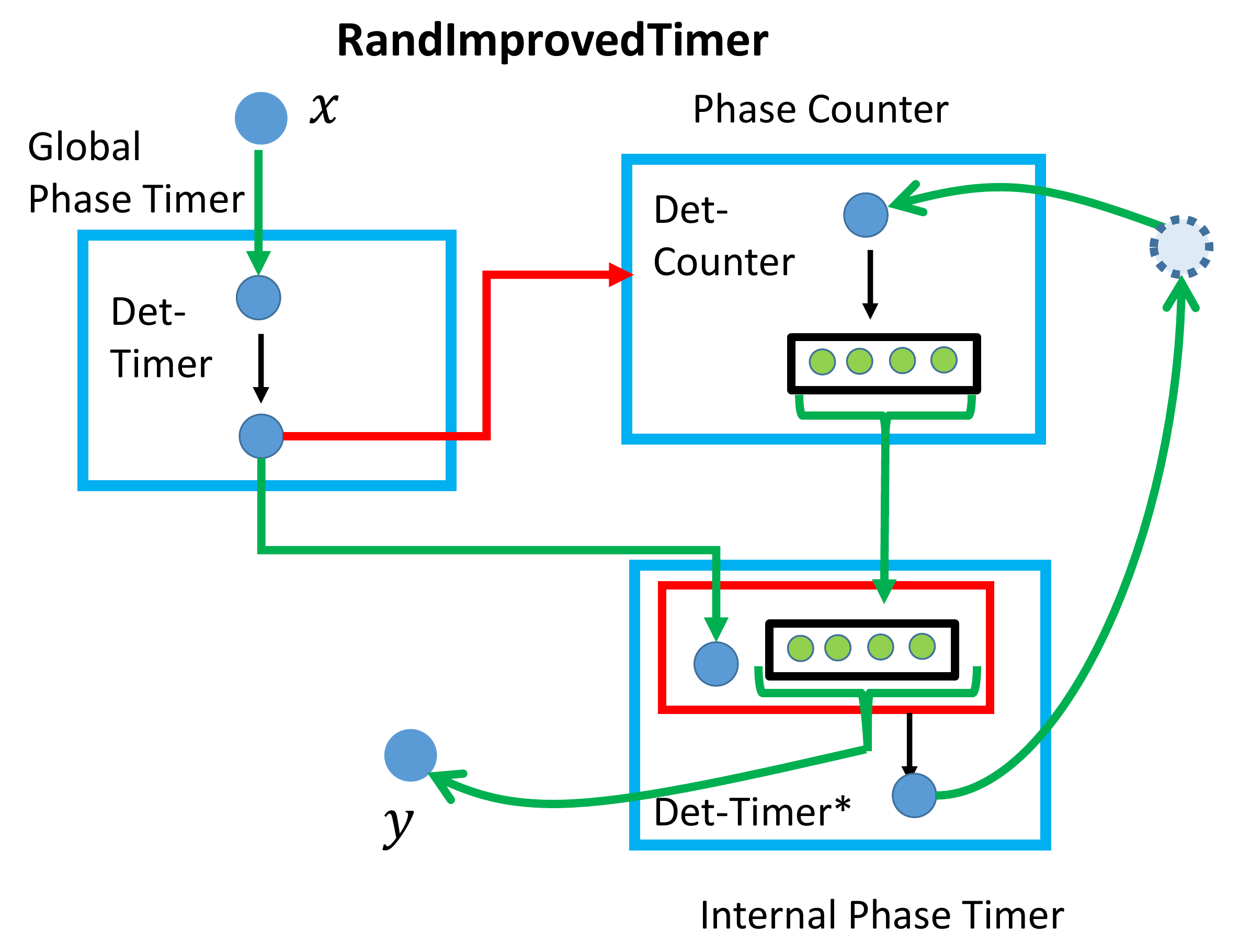}
\end{center}
\caption{ 
Schematic description of the randomized timer. In each module only the input layer and the output layer are shown. Excitatory (inhibitory) relations are shown in green (red) arrows. Each module is deterministic and has $\Theta(\log\log 1/\delta)$ threshold gates. The lower right module (\emph{Internal-Phase-Timer}) uses the variant of the deterministic neural timer in which the time parameter is softly encoded in the input layer. This is crucial as the length of the $(i+1)^{th}$ active phase depends on the spike counts of $a^*$ in phase $i$. This value is encoded by the output layer of the \emph{Phase-Counter} module at the end of phase $i$. In contrast, the \emph{Global-Phase-Timer} module uses the standard neural timer network (hard-wired), as the length of each phase is fixed. \label{fig:loglog-random}}	
\end{figure} 

\paragraph{Correctness.} For simplicity we begin by showing the correctness of the construction assuming that there is a single firing of the input $x$ during a period of $2t$ rounds.
Taking care of the general case requires minor modifications that are described at the end of this section.

Our goal is to show that each \emph{phase} of the $\RandImprovedTimer$ network is equivalent to a \emph{round} in the $\RandBasic$ network. Toward that goal, we start by showing that the length of the active part of phase $i$ has the same distribution as the number of neurons $n_{i-1}$ that fire in round $i-1$ in $\RandBasic(t',\delta)$, where $t'=t/\ell'$. 
In the $\RandBasic(t',\delta)$ construction, let $\bar{B}_i$ be a random variable indicating the event that there exists a neuron $a \in A$ which fired in round $\tau \leq i$ but $a$ as well as $x$ did not fire in round $\tau-1$.
Similarly, for the $\RandImprovedTimer(t, \delta)$ construction, let $\bar{B}_i'$ be a random variable indicating the event that there exists a phase $\tau \leq i$, where neuron $a^*$ fired in an inactive round of phase $\tau$. 
Note that in both constructions, the probability that $a^*$ fired in an inactive round, and the probability that $a \in A$ fired given that it did not fire in the previous round is identical. Moreover, within a window of $ \tau = \poly(t)$ rounds, by union bound both probabilities $\Pr[B_{\tau}]$ and $\Pr[B'_\tau]$ are at most $\delta/2$. 

Let $Y_i$ be a random variable for the number of neurons that fired in the $i^{th}$ round in $\RandBasic(t',\delta)$, and let $X_i$ be the random variable for the number of rounds $a^*$ fired during phase $i$ in $\RandImprovedTimer(t, \delta)$. In both constructions we assume that the input neuron $x$ fired only in round $0$.
\begin{claim}
For any $k \geq 0$ and $i \geq 1$, $\Pr[X_i = k ~\mid~ \bar{B'}_i] = \Pr[Y_i = k ~\mid~ \bar{B}_i]$.
\end{claim}
\begin{proof}
By induction on $i$. For $i=1$, given $\bar{B}_1$, $\bar{B'}_1$, the random variable $X_1$ as well as $Y_1$ are the sum of $\ell$ independent Bernoulli variables with probability $1-1/t'$ and therefore $X_1 = Y_1$. Assume $\Pr[X_i = k ~\mid~ \bar{B}_i] = \Pr[Y_i = k ~\mid~ \bar{B}_i]$ and we will show the equivalence for $i+1$. First recall that in the $\RandBasic(t',\delta)$ construction, each $a \in A$  fires with probability $1-1/t'$ given that it fired in the previous round. Moreover, conditioning on $\bar{B}_{i+1}$, given that $a$ did not fire in round $i$, it does not fire in round $i+1$ as well. Thus, for any $k , j$ it holds that $\Pr[Y_{i+1}=k ~\mid~ Y_i=j, \bar{B}_{i+1}] = \binom{j}{k}(1- 1/t')^k\cdot (1/t')^{j-k}$ (i.e., a binomial distribution). Similarly, in the $\RandImprovedTimer(t,\delta)$ construction, since we assumed that $a^*$ fires only in the active rounds of each phase, given that $a^*$ fired $j$ times in phase $i$, in phase $i+1$ it holds that $\Pr[X_{i+1}=k ~\mid~ X_i=j, \bar{B'}_{i+1}] = \binom{j}{k}(1- 1/t')^k\cdot (1/t')^{j-k}$. 
By the law of total probability we conclude that 
\begin{eqnarray}
\Pr[X_{i+1}&=&k ~\mid~ \bar{B'}_{i+1}] = \sum_{j=0}^{\ell} \Pr[X_{i+1}=k ~\mid~ X_i=j, \bar{B'}_{i+1}] \cdot \Pr[X_i=j ~\mid~  \bar{B'}_{i+1}] \nonumber
\\&=& \sum_{j=0}^{\ell} \Pr[X_{i+1}=k ~\mid~ X_i=j, \bar{B'}_{i+1}] \cdot \Pr[Y_i=j ~\mid~ \bar{B}_i]  \nonumber
\\&=& \sum_{j=0}^{\ell} \binom{j}{k}(1- 1/t')^k\cdot (1/t')^{j-k} \cdot \Pr[Y_i=j ~\mid~ \bar{B}_i] \nonumber
\\&=& \sum_{j=0}^{\ell} \Pr[Y_{i+1}=k ~\mid~ Y_i=j, \bar{B}_{i+1}] \cdot \Pr[Y_i=j ~\mid~ \bar{B}_i] =  \Pr[Y_{i+1}=k ~\mid~ \bar{B}_{i+1}], \nonumber
\end{eqnarray}
where the second equality is due to the induction assumption.
\end{proof}
Hence, by the correctness of the network $\RandBasic(t',\delta)$, with probability at least $1- \delta$ the neuron $a^*$ fires at least $\ell/2e$ times in each of the first $t'$ phases. Since every phase consists of $\ell' = \ell+ \log \ell$ rounds, $y$ fires for at least $\ell'\cdot t' = t$ rounds w.h.p.
On the other hand, with probability at most $\delta/2$ the neuron $a^*$ fires in an inactive round during one of the first $2 t'$ phases. Given that $a^*$ fired only in active-rounds, we conclude that with probability at most $\delta/2$ the output $y$ fires for at least $2t'$ phases. All together, with probability at least $1- \delta$ the output $y$ stops firing by round $2t'\cdot \ell' = 2t$.

Finally, we describe the small modifications needed to handle the case where $x$ fires several times within a window of $2t$ rounds. Upon any firing of $x$, all modules get reset, and a new counting starts. To implement the reset, we connect the input neuron $x$ to two additional neurons, an inhibitor neuron $x_1$, and an excitatory neuron $x_2$ where $w(x,x_1)=w(x,x_2)=1$ with thresholds $b(x_1)= b(x_2) =1$. 
The inhibitor $x_1$ has outgoing edges to all auxiliary neurons in the network with weight $-4$. The excitatory neuron $x_2$ has outgoing edges to the input of \emph{Global-Phase-Timer} and the time input of the \emph{Internal-Phase-Timer}, such that the decimal value is equal to $\ell$ with weights $6$. Thus, after one round of cleaning-up, the network starts to account the last firing of $x$. The output $y$ has incoming edges from $x$ and $x_2$ each with weight $w(x,y)= w(x_1,y)= \frac{\ell}{2e}$, this makes $y$ fire during the reset period.

\subsection{A Matching Lower Bound}
We now turn to show a matching lower bound with randomized spiking neurons. 
Assume towards contradiction there exists a randomized neural timer $\cN$ for a given time parameter $t$ with $N=o(\log\log 1/\delta)$ neurons that succeeds with probability at least $1-\delta$. This implies that once $x$ fired, $y$ fires for $t$ consecutive round with probability $1-\delta$. 
Moreover, there exists some constant $c \geq 2$ such that $y$ stops firing after $(c-1) \cdot t$ rounds w.p $1-\delta$. Throwout the proof, we assume w.l.g that $x$ fired in round $0$.
Recall that the state of the network in time $\tau$ denoted as $s_\tau$ can be described as an $N$-length binary vector.
Since we have $N$ many neurons, the number of distinct states (or configurations) is bounded by $S=2^N=o(\log 1/\delta)$. 
We start by establishing useful auxiliary claims.

We first claim that because we have relatively small number of states, and the memoryless property discussed in section~\ref{sec:perlim} in every window of $t$ rounds there exists a state that occurs at least twice (and with sufficient distance).
Let $s^*$ be a state for which the probability that there exist rounds $t', t'' \leq t$ such that $\frac{t}{3\cdot S} \leq t'-t'' \leq t$ and $s_{t'}=s_{t''}=s^*$ is at least $1/S$.
\begin{claim} \label{claim: low}
	There exists such a state $s^*$.
\end{claim}
\begin{proof}
	Note that because $N = o(\log \log 1/\delta)$ and $1/\delta \leq 2^{\poly(t)}$ it holds that $\frac{t}{3\cdot S} \geq 1$.
	We partition the interval $[0,t]$ into $2\cdot (S+1)$ balanced intervals, each of size $t/2(S+1)$. 
	Because we have only $S$ different states, in every execution of the network for $t$ many rounds, there must be a state that occurs in at least two even intervals. 
	Thus, there exists a state $s^{*}$ for which the probability that $s^{*}$ occurred in two even intervals is at least $1/S$. Because each interval is of size $t/2(S+1) \geq t/3S$ we conclude that the claim holds.	
\end{proof}
Next we use the assumption that with probability at least $1 - \delta$ the output $y$ fires in rounds $[0,t]$ as well as the memoryless property to show that given that state $s^*$ occurred in round $t'$, with a sufficiently large probability, $s^*$ occurs again with a long enough interval, and $y$ fires in all rounds between the two occurrences of $s^*$.
Let $p(t')=\Pr[\exists t'' \in [t'+t/(3S), t'+t], s_{t''}=s^* \mbox{~~and~~} y^{t^*}=1 ~~\forall t^*\in [t',t''] ~\mid~ s_{t'}=s^*]$. By the memoryless property, we have: 
\begin{Observation}
	$p(t')=p(t'')$ for every $t',t''$.
\end{Observation}
Define $p^*=p(1)= p(t')$ for any round $t'$. The next claim shows that $p^*$ is sufficiently large. 
\begin{claim} \label{clm: low3}
	$p^*\geq 1/S-\delta$.
\end{claim}
\begin{proof}
	Let $A$ be an indicator random variable for the event that there exist $0<t',t''<t$ such that $t''-t'\in [t/3S, t]$, $s_{t'}=s_{t''}=s^*$. 
	Let $B$ be the indicator random variable for the event that there exists $t^* \in [0,t]$ such that $y^{t^*}=0$. 
	By Claim~\ref{claim: low}, $\Pr[A] \geq 1/S$, and by the success guarantee of the network, $\Pr[B] \leq \delta$. Hence, by union bound, we get 
	$\Pr[A \wedge \bar{B}] \geq 1/S-\delta~.$
	
	Let $A(t',t'')$ be the indicator random variable for the event that $s_{t'}=s_{t''}=s^*$ and $y^{t^*}=1$ for every $t^* \in [t',t'']$. Let $F(t')$ be the indicator random variable for the event that $s^*$ appears in round $t'$ for the first time. Hence, we get
	\begin{eqnarray}
	1/S-\delta &\leq& \Pr[A \wedge \bar{B}] \leq \sum_{0<t' <t-t/(3S)}\Pr[F(t') \wedge \left(\exists t'' \in [t'+t/3S,t] \mbox{~~s.t~~} s_{t''}=s^*\right) \wedge \bar{B}] \nonumber
	\\& \leq& 
	\sum_{0<t' <t-t/(3S)}\Pr[F(t') \wedge \left(\exists t'' \in [t'+t/3S,t] \mbox{~~s.t~~} A(t',t'')=1 \right)] \nonumber
	\\&= & \sum_{0<t' <t-t/(3S)}\Pr[F(t')] \cdot \Pr[\left(\exists t'' \in [t'+t/3S,t] \mbox{~~s.t~~} A(t',t'')=1 \right) ~\mid~ F(t')]  \nonumber
	\\&= &
	\sum_{0<t' <t-t/(3S)}\Pr[F(t')] \cdot \Pr[\left(\exists t'' \in [t'+t/3S,t] \mbox{~~s.t~~} A(t',t'')=1 \right) ~\mid~ s_{t'}=s^*] \nonumber
	\\& = &
	\sum_{0<t' <t-t/(3S)}\Pr[F(t')] \cdot p^* \leq p^*, \nonumber
	\end{eqnarray}
	where the second inequality is by union bound over all possibilities for event $A$.
	The third equation is due to the memoryless property, the probability that the event occurred conditioning on $F(t')$, is equivalent to conditioning on $s_{t'}=s^*$. 
	The last equality follows by summing over a set of disjoint events $F(t')$.
\end{proof}
We are now ready to complete the proof of Theorem \ref{lem:lower-bound-rand}. 

\paragraph{Proof of Thm.~\ref{lem:lower-bound-rand}.} 
We bound the probability that $y$ fires in each of the first $c \cdot t$ rounds.
Let $C$ be the indicator random variable for the event that there exists a sequence of rounds $0 = \tau_0 <\tau_1< \tau_2<\dots<\tau_{3c \cdot S}$ such that for every $i \geq 1$, it holds that: 
\begin{itemize}
	\item $y^{t^*}=1$ for all $t^{*} \in [\tau_{i-1}, \tau_{i}]$.
	\item  $ s_{\tau_i}=s^*$.
	\item $\tau_i-\tau_{i+1} \in [t/3S,t]$.
\end{itemize} 
Note that because $\tau_{i+1}-\tau_{i} \geq t/3S$, it holds that $\tau_{3c \cdot S} \geq c \cdot t$.
Hence, the probability that $y$ fires in each of the first $c \cdot t$ rounds is at least $\Pr[C]$.
Next, we calculate the probability of the event $C$. 
Recall that given that  $s_{\tau_{i}}=s^*$, the probability there exists a round $\tau_{i+1} \in [\tau_i+t/3S,\tau_i+t]$ for which $A(\tau_{i}, \tau_{i+1})$ is equal to $p(\tau_{i})= p^*$. Moreover, by Claim~\ref{claim: low} and the success guarantee, the probability there exists $0 < \tau_1 < t$ such that $A(0,\tau_1)$ is at least $1/S-\delta$. By Claim~\ref{clm: low3} and the memoryless property, we have:
\begin{eqnarray}
\Pr[C] &=&  \Pr[\exists \tau_1 < t ~s.t~ A(0,\tau_1)]\cdot \prod_{i=1}^{3c \cdot S-1}\Pr[\exists \tau_{i+1}\in [\tau_i+t/3S,\tau_i+t] ~s.t~ A(\tau_{i},\tau_{i+1}) ~\mid~ s_{\tau_i = s^*}] \nonumber 
\\& \geq& (1/S-\delta) \prod_{i=1}^{3c \cdot S-1}p^* \geq \left(1/S-\delta\right)^{3 c  S}~. \nonumber
\end{eqnarray}
Taking $N \leq (\log \log 1/\delta - \log\log\log 1/\delta) -\log 6c = \log(\frac{\log 1/\delta}{\log \log 1/\delta}) -\log 6c$, the number of different states is bounded by  $S < \frac{ \log 1/\delta}{6 c \cdot \log \log 1/ \delta}$. Thus the network fails with probability at least 
$$\left(1/S- \delta\right)^{3c \cdot S} > \left(\frac{6 \cdot c \cdot \log \log 1/ \delta}{\log 1/\delta}\cdot \frac{1}{2}\right)^{\frac{\log 1/ \delta}{\log \log 1/\delta}} > \delta~,$$ 
in contradiction to the success guarantee of at least $1 - \delta$.

\section{Applications to Synchronizers}

\textbf{The Asynchronous Setting.} In this setting, the neural network $\cN =\langle X,Z,Y,A, w,b \rangle$ also specifies a response latency function $\ell: A \to \mathbb{N}_{>0}$. For ease of notation, we normalize all latency values such that $\min_{e \in A} \ell(e)=1$ and denote the maximum response latency by $L=\max_{e \in A} \ell(e)$. 
Supported by biological evidence \cite{ikeda2006autapses}, we assume that self-loop edges (a.k.a. autapses) have the minimal latency in the network, that $\ell((u,u))=1$ for self-edges $(u,u)$. This assumption is crucial in our design\footnote{In a follow-up work, we actually show that this assumption is necessary for the existence of syncrnoizers even when $L=2$.}. Indeed the exceptional short latency of self-loop edges has been shown to play a critical role in biological network synchronization \cite{ma2015autapse,fan2018autapses}.
The dynamic proceeds in synchronous rounds and phases. The length of a round corresponds to the minimum edge latency, this is why we normalize the latency values so that $\min_{e \in A} \ell(e)=1$. If neuron $u$ fires in round $\tau$, its endpoint $v$ receives $u$'s signal in round $\tau+\ell(e)$. Formally, a neuron $u$ fires in round $\tau$ with probability $p(u,\tau)$:
\begin{align}
\label{eq:potentialOut-async}
\pot(u,\tau)= \hspace{-.65em}  \sum_{v \in X\cup Z \cup Y}w_{v,u}\cdot v^{\tau-\ell(u,v)} -b(u) 
\text{ and }p(u,\tau)=\frac{1}{1+e^{-\frac{\pot(u,\tau)}{\lambda}}}
\end{align}

\paragraph{Synchronizer.}
A synchronizer $\nu$ is an algorithm that gets as input a network $\cN_{\sync}$ and outputs a network $\cN_{\async}= \nu(\cN_\sync)$ such that $V(\cN_{\sync}) \subseteq V(\cN_{\async})$ where $V(\cN)$ denotes the neurons of a network $\cN$. The network $\cN_{\async}$ works in the asynchronous setting and should have \emph{similar execution} to $\cN_{\sync}$ in the sense that
for every neuron $v \in V(\cN_{\sync})$, the firing pattern of $v$ in the asynchronous network should be similar to the one in the synchronous network. The output network $\cN_{\async}$ simulates each round of the network $\cN_{\sync}$ as a \emph{phase}. 
 
\begin{definition}[Pulse Generator and Phases]
A \emph{pulse generator} is a module that fires to declare the end of each phase.
Denote by $t(v,p)$ the (global) round in which neuron $v$ receives the $p^{th}$ spike from the pulse generator. We say that $v$ is in phase $p$ during all rounds $\tau \in [t(v,p-1), t(v,p)]$. 
\end{definition} 
\begin{definition}[Similar Execution (Deterministic Networks)]
The \emph{synchronous} execution $\Pi_{\sync}$ of a deterministic network $\cN_{\sync}$ is specified by a list of states $\Pi_{\sync}=\{\sigma_1,\ldots, \}$ where each $\sigma_i$ is a binary vector describing the firing status of the neurons in round $i$. The \emph{asynchronous} execution of network $\cN_{\async}$ denoted by $\Pi_{\async}$ is defined analogously only when applying the asynchronous dynamics (of Eq. (\ref{eq:potentialOut-async})).  
The execution $\Pi_{\async}$ is divided into phases of fixed length. 
The networks $\cN_{\sync}$ and $\cN_{\async}$ have a \emph{similar execution} if $V(\cN_{\sync}) \subseteq V(\cN_{\async})$, and in addition, a neuron $v \in V(\cN_{\sync})$ fires in round $p$ in the execution $\Pi_{\sync}$  iff $v$ fires during phase $p$ in $\Pi_{\async}$.
\end{definition}
\vspace{-5pt}
For simplicity of explanation, we assume that the network $\cN_{\sync}$ is deterministic. However, our scheme can easily capture randomized networks as well (i.e., by fixing the random bits in the synchronized simulation and feeding it to the async. one). 
\subsection{Extension for Randomized Networks}\label{sec:rand-sync}
 For networks $\cN_{\sync}$ that contain also probabilistic threshold gates, the notion of similar execution is defined as follows. Consider a fixed execution $\Pi_{\sync}$ of the network $\cN_{\sync}$. In each round of simulating $\cN_{\sync}$, the spiking neurons flip a coin with probability that depends on their potential. Once we fix those random coins used by the neurons in the execution $\Pi_{\sync}$, the process becomes deterministic. Formally, for every round $p$ and neuron $v$, let $R(v,p)$ be the set of random coins used by the neuron $v$ in round $p$ in the execution $\Pi_{\sync}$. The firing decision of $v$ in round $p$ is fully determined given those bits. The asynchronous network $\cN_{\async}$ contains a set of neurons $V'$ that are analogous to the neurons in $\cN_{\sync}$ and an additional set of deterministic threshold gates.
When simulating this network, the neurons in $V'$ will use the \emph{same} random coins as those used by their corresponding neurons in $\Pi_{\sync}$: in each phase $p$ in the execution $\Pi_{\async}$, the neuron $v$ will be given the bits $R(v,p)$ and will base its firing decision using a deterministic function of its current potential, bias value and $R(v,p)$. This allows us to restrict attention to deterministic networks\footnote{Where the neurons in those networks are not necessarily threshold gates, but rather base their firing decision using \emph{some} deterministic function}.  

\textbf{The Challenge.}  
Consider a network of a threshold gate $z$ with two incoming inputs: an excitatory neuron $x$, and an inhibitory neuron $y$. The weights are set such that $z$ computes $X  \wedge \bar{Y}$ thus it fires in round $\tau$ if $x$ fired in round $\tau-1$ and $y$ did not fire. Implementing an $X  \wedge \bar{Y}$ gate in the asynchronous setting is quite tricky. In the case where both $x$ and $y$ fire in round $\tau$, in the synchronous network, $z$ should not fire in round $\tau+1$. However, in the asynchronous setting, if $\ell(x,z)<\ell(y,z)$, then $z$ will mistakenly fire in round $\tau+\ell(x,z)$. This illustrates the need of enforcing a \emph{delay} in the asynchronous simulation: the neurons should attempt firing only after receiving \emph{all} their inputs from the previous phase. We handle this by introducing a pulse-generator module, that announces when it is safe to attempt firing.

To illustrate another source of challenge, consider the asynchronous implementation of an AND-gate $X  \wedge  Y$. 
If both $x$ and $y$ fire in round $\tau$, then $z$ fires in round $\tau+1$ in the synchronous setting. However, if the latencies of the edges $\ell(x,z)$ and $\ell(y,z)$ are distinct, $z$ receives the spike from $x$ and $y$ in \emph{different} rounds, thus preventing the firing of $z$. Recall, that $z$ has no memory, and thus  its firing decision is based only on the potential level in the previous round. 
To overcome this hurdle, in the transformed network, each neuron in the original synchronous network is augmented with $3$ copy-neurons, some of which have self-loops. Since self-loops have latency $1$, once a neuron with a self-loop fires, it fires in the next round as well. This will make sure that the firing states of $x$ and $y$ are kept on being presented to $z$ for sufficiently many rounds, which guarantees the existence of a round where both spikes arrive. 

While solving one problem, introducing self-loops into the system brings along other troubles. Clearly, we would not want the neurons to fire forever, and at some point, those neurons should get inhibition to allow the beginning of a new phase. This calls for a delicate \emph{reset} mechanism that cleans up the old firing states at the end of each phase, only after their values have already being used. Our final solution consists of global synchronization modules (e.g., pulse-generator, reset modules) that are inter-connected to a modified version of the synchronous network. 
Before explaining those constructions, we start by providing a modified neural timer $\DetTimer_{\async}$ adapted to asynchronous setting. This timer will be the basic building block in our global synchronization infrastructures. 

\paragraph{Asynchronous Analog of $\DetTimer$.}
A basic building block in our construction is a variant of $\DetTimer$ to the asynchronous setting. Observe that the $\DetTimer$ implementation of Sec. \ref{sec:det-timer} might fail miserably in the asynchronous setting, e.g., when the edges $(a_{i-1,2}, a_{i,2})$ have latency $2$ for every $i \geq 2$, and the remaining edges have latency $1$, the timer will stop counting after $\Theta(\log t)$ rounds, rather than after $t$ rounds. In Appendix~\ref{app:det-timer-async}, we show:
\vspace{-5pt}
\begin{lemma}\label{lem:det-timer-async}[Neural Timer in the Asynchronous Setting]
	For a given time parameter $t$, there exists a deterministic network $\DetTimer_{\async}$ with $O(L \cdot \log t)$ neurons, satisfying that in the asynchronous setting with maximum latency $L$, the output neuron fires at least $\Theta(t)$ rounds, and at most $\Theta(L \cdot t)$ rounds after each firing of the input neuron.
\end{lemma}

\paragraph{Description of the Syncronizer.} 
The construction has two parts: a global infrastructure, that can be used to synchronize many networks\footnote{It is indeed believed that the neural brain has centers of synchronization.}, and an adaptation of the given network $\cN_{\sync}$ into a network $\cN_{\async}$.  
The global infrastructures consists of the following modules:
\begin{itemize}
\item A \emph{pulse generator} $PG$ implemented by $\DetTimer_{\async}$ with time parameter $\Theta(L^3)$. 
\item A \emph{reset} module $R_1$ implemented by a directed chain of $\Theta(L)$ neurons \footnote{Each neuron in the chain has an incoming edge from its preceding neuron with weight $1$ and threshold $1$.} with input from the output neuron of the $PG$ module.
\item A \emph{delay} module $D$ implemented by $\DetTimer_{\async}$ with time parameter $\Theta(L^2)$ and input from the output of of the $PG$ module.
\item Another \emph{reset} module $R_2$ implemented by a chain of $\Theta(L)$ neurons with input from $D$.
\end{itemize}
\vspace{-5pt}
The heart of the construction is the pulse-generator that fires once within a fixed number of $\ell \in [\Theta(L^3),\Theta(L^4)]$ rounds,  and invokes a cascade of activities at the end of each phase. When its output neuron $g$ fires, it activates the reset and the delay modules, $R_1$ and $D$. The second reset module $R_2$ will be activated by the delay module $D$. Both reset modules $R_1$ and $R_2$ are implemented by chains of length $L$, with the last neuron on these chains being an \emph{inhibitor} neuron. The role of the reset modules is to \emph{erase} the firing states of some neurons (in $\cN_{\async}$) from the previous phase, hence their output neuron is an inhibitor. The timing of this clean-up is very delicate, and therefore the reset modules are separated by a delay module that prevents a premature operation.  The total number of neurons in these global modules is $O(L \cdot \log L)$. 
We next consider the specific modifications to the synchronous network $\cN_{\sync}$ (see Fig. ~\ref{fig:async-global}).
\vspace{-5pt}
\paragraph{Modifications to the Network $\cN_{\sync}$.} 
The input layer and output layer in $N_{\async}$ are \emph{exactly} as in $\cN_{\sync}$. 
We will now focus on the set of \emph{auxiliary} neurons $V$ in $\cN_{\sync}$. 
In the network $\cN_{\async}$, each
$v \in V$ is augmented by three additional neurons $v_{\inn}, v_{\delaynode}$ and $v_{\outt}$. 
The incoming (resp., outgoing) neighbors to $v_{\inn}$ (resp., $v_{\outt}$) are the out-copies (resp., in-copies) of all incoming (resp., outgoing) neighboring neurons of $v$. The neurons $v_{\inn}, v, v_{\delaynode}$ and $v_{\outt}$ are connected by a directed chain (in this order). Both $v_{\delaynode}$ and $v_{\outt}$ have self-loops. 

In case where the original network $\cN_{\sync}$ contains spiking neurons, the neuron $v_{\inn}$ will be given the exact same firing function as $v$ in $\Pi_{\sync}$. That is, in phase $p$, $v_{\inn}$ will be given the random coins\footnote{I.e., the random coins that are used to simulate the firing decision of $v$.} used by $v$ in round $p$ in $\Pi_{\sync}$. The other neurons $v, v_{\delaynode}$ and $v_{\outt}$ are deterministic threshold gates. 
%
	%
	%
The role of the \emph{out-copy} $v_{\outt}$ is to \emph{keep on presenting} the firing status of $v$ from the previous phase $p-1$ throughout the rounds of phase $p$. This is achieved through their self-loops. 
The role of the \emph{in-copy} $v_{\inn}$ is to simulate the firing behavior of $v$ in phase $p$. We will make sure that $v_{\inn}$ fires in phase $p$ only if $v$ fires in round $p$ in $\Pi_{\sync}$. For this reason, we set the incoming edge weights of $v_{\inn}$ as well as its bias to be exactly the same as that of $v$ in $\cN_{\sync}$. 
The neuron $v$ is an AND gate of its in-copy $v_{\inn}$ and the $PG$ output $g$. Thus, we will make sure that $v$ fires at the \emph{end} of phase $p$ only if $v_{\inn}$ fires in this phase as well. 
The role of the delay copy $v_{\delaynode}$ is to delay the update of $v_{\outt}$ to the up-to-date firing state of $v$ (in phase $p$). 
Since both neurons $v_{\delaynode}$ and $v_{\outt}$ have self-loops, at the end of each phase, we need to carefully reset their values (through inhibition). This is the role of the reset modules $R_1$ and $R_2$. Specifically, the reset module $R_1$ operated by the pulse-generator inhibits $v_{\outt}$. The second reset module $R_2$ inhibits the delay neuron $v_{\delaynode}$ only after we can be certain that its value has already being ``copied" to $v_{\outt}$. 
Finally, we describe the connections of the neuron $v_{\outt}$. The neuron $v_{\outt}$ has an incoming edge from the reset module $R_1$ with a super-large weight. This makes sure that when the reset module is activated, $v_{\outt}$ will be inhibited shortly after. In addition, it has a self-loop also of large weight (yet smaller than the inhibition edge) that makes sure that if $v_{\outt}$ fires in a given round, and the reset module $R_1$ is not active, $v_{\outt}$ also fires in the next round. Lastly, if $v_{\outt}$ did not fire in the previous round, then it fires when receiving the spikes from \emph{both} the delay module and from the delay copy $v_{\delaynode}$. This will make sure that the firing state of $v_{\delaynode}$ will be copied to $v_{\outt}$ only after the output of the delay module $D$ fires. 

\begin{figure}[h]
\centering
\vspace{-10pt}
\includegraphics[width=0.5\textwidth]{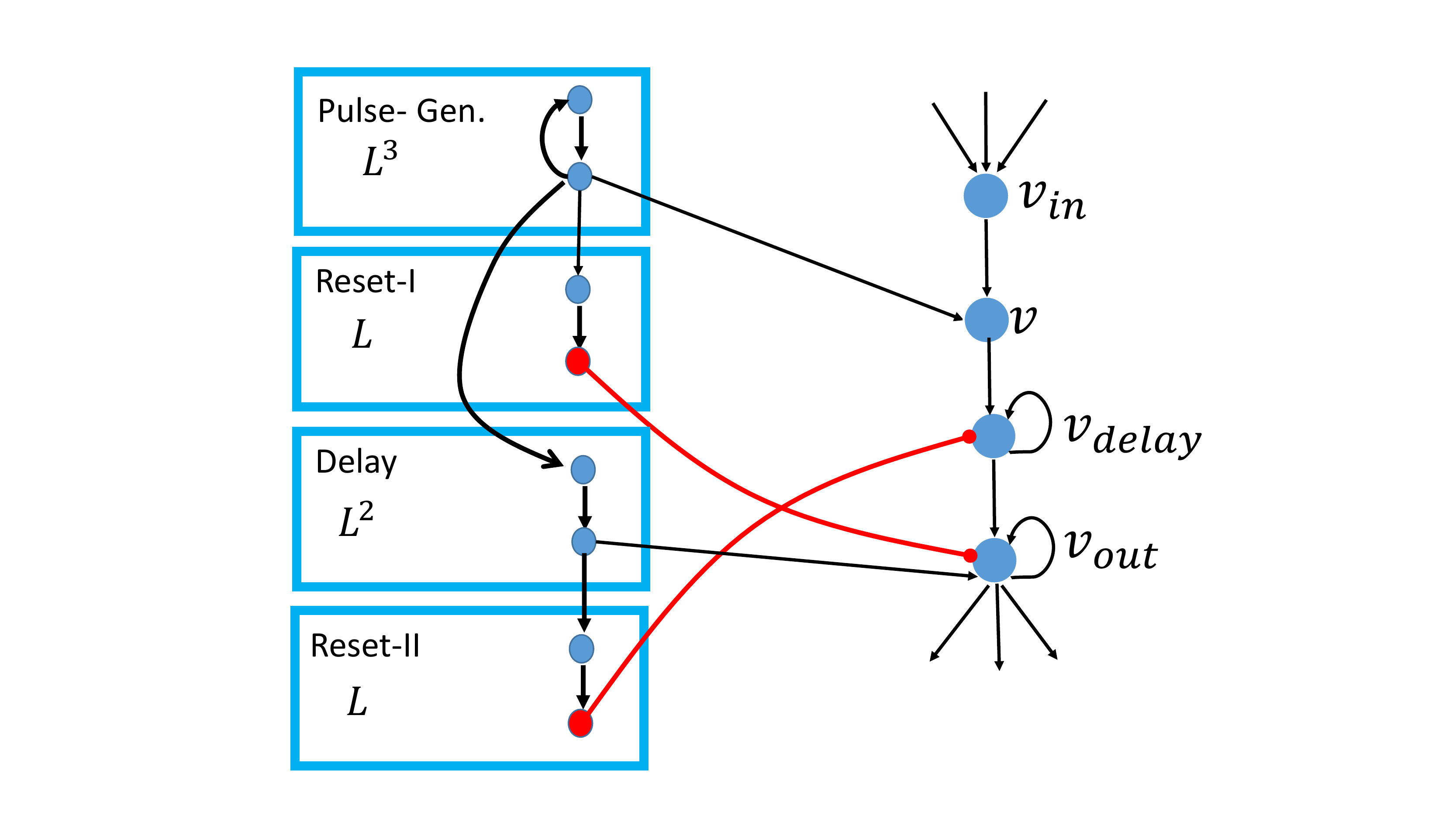}
\vspace{-5pt}
	\caption{{\footnotesize Illustration of the syncronizer modules. Left: global modules implemented by neural timers. Right: a neuron $v \in N_{\sync}$ augmented by three additional neurons that interact with the global modules.}
\label{fig:async-global}}
\end{figure}

\subsection{Analysis of the Syncronizers.}
Throughout, we fix a synchronous execution $\Pi_{\sync}$ and an asynchronous execution $\Pi_{\async}$. For every round $p$, recall that $V^+_{\sync}(p)$ is the set of neurons that fire in round $p$ in $\Pi_{\sync}$ (i.e., the neurons with positive entries in $\sigma_p$).
In our simulation, we will make sure that each $v$ in $\cN_{\async}$ has the same firing pattern as its copy in $\cN_{\sync}$.

\begin{Observation} \label{obs: receive_together}
Consider a neuron $v$ with incoming neighbors $u_1, \ldots, u_k$. 
If there is a round $\tau$ such that $u_1, \dots, u_k$ fire in each round $\tau' \geq \tau$, $v$ fires in every round $\tau'' \geq \tau+\max_{u_i}\ell(v,u_{i})$.
\end{Observation}

\begin{lemma}
The networks $\cN_{\sync}$ and $\nu(\cN_{\sync})=\cN_{\async}$ have similar executions. 
\end{lemma}
\begin{proof}
We will show by induction on $p$ that $V^+_{\sync}(p)=V^+_{\async}(p)$.
For $p=1$, let $ V^+_{\sync}(0)$ be the neurons that fired at the beginning of the simulation in round $0$. 
We now show that every neuron $v \in V$ fires at the end of phase $1$ iff $v \in V^+_{\sync}(1)$. Without loss of generality, assume that $g$ fired at the end phase $0$ and begins the simulation in round $0$. We begin with the following claim.
\begin{claim}
For every $u \in V$, for its in-copy $u_{\inn}$ there is a round $\tau_{u} \leq c_2 L^3+ L$ in which all its incoming neighbors in $V^+_{\sync}(0)$ fire (and the remaining neighbors do not fire), for a constant $c_2$.
\end{claim}
\begin{proof}
	We first show that for $v \in V^+_{\sync}(0)$, the out-copy $v_{\outt}$ fires when it receives a signal from the delay module $D$. Because each edge has latency at most $L$, by round $L$, neuron $v$ has fired. Since the delay neuron $v_{\delaynode}$ has a self loop (with latency one), it starts firing in every round starting round $\tau_{d} \in [2, 2L]$ (until it is inhibited by the reset module $R_2$). Recall that the out-copy $v_{\outt}$ is connected to the delay module $D$, and fires only when receiving a spike from both the output neuron of $D$ \emph{and} the delay-neuron $v_{\delaynode}$. 
	We claim that $v_{\outt}$ receives a signal from $D$ and starts firing \emph{after} it gets a reset from $R_1$. The reset module $R_1$ receives the signal from $g$ by round $L$ and starts counting $L$ rounds. Thus, the output neuron of $R_1$ fires in some round $\tau'_{r_1} \in [L + 1, L^2 + L]$. This insures that by round $L^2+2L$ the neuron $v_{\outt}$ is inhibited by the output of $R_1$. The delay module $D$ is implemented by $\DetTimer_{\async}$ with time parameter $2L^2$. Therefore, the output neuron of $D$ fires in round $\tau_{D} \in [2L^2, 10L^3]$, ensuring that it fires only \emph{after} $v_{\outt}$ has been reset by the module $R_1$.  
	Moreover, the reset module $R_2$ counts $L$ rounds after receiving a signal from $D$. This ensures that the inhibitory output of $R_2$ starts inhibiting $v_{\delaynode}$ only \emph{after} $v_{\outt}$ has received the signal from $D$ in round $\tau_{\outt}$. Overall, we conclude that $v_{\outt}$ fires in round $\tau_{\outt}\in [c_1 \cdot L^2, c_2 \cdot L^3]$, for some constants $c_1,c_2$. Due to the self loop, $v_{\outt}$ also fires in each round $\tau'' \geq \tau_{\outt}$ in that phase. As a result for every $u \in V$, its in-copy $u_{\inn}$ has a round $\tau_{u} \leq c_2 L^3+ L$ in which all its incoming neighbors in $V^+_{\sync}(0)$ fire. Note that for every neuron $v \notin V^+_{\sync}(0)$, non of its copy neurons $v_{\outt}, v_{\delaynode}$ fire during the phase.
\end{proof}
Hence, $u_{\inn}$ start firing in round $\tau_{u}$ only if $u$ fires in round $1$ in $\Pi_{\sync}$, i.e., if $u \in V^+_{\sync}(1)$.  We set the pulse-generator with time parameter $c_3 \cdot L^3$ for a large enough $c_3$ such that $c_3 \cdot L^3 > c_2 L^3+ 2L$. Since the out-copies keep on presenting the firing states of phase $0$, $u_{\inn}$ continues to fire in the last $L$ rounds of the phase. Thus, when the pulse-generator spikes again, the neurons in $V^+_{\sync}(1)$ indeed fire as both $g$ and $v_{\inn}$ fired in the previous rounds.

Next, we assume that $V^+_{\sync}(p)=V^+_{\async}(p)$ and consider phase $p+1$. Let $\tau^*$ be the round that the $PG$ fired at the end of phase $p$. We first show the following.
\begin{claim} \label{clm:delay-neuron}
	For every $v \in V$, the neuron $v_{\delaynode}$ starts firing by round $\tau^*+2L$, iff $v \in V^+_{\sync}(p)$.
\end{claim}
\begin{proof}
 Recall that all delay copies are inhibited by the reset module $R_2$ at most $L^2+2L$ rounds after the delay module $D$ has fired. We choose the time parameter of the $PG$ to be large enough such that this occurs before the next pulse of $PG$ in round $\tau^*$. Hence, when phase $p$ ended in round $\tau^*$, all delay copies $v_{\delaynode}$ are idle. Because each edge has latency of at most $L$, by round $\tau^* + L$, all the neurons in $V^+_{\sync}(p)$ have fired (and by the assumption other neurons did not fire during phase $p$). As a result, the neuron $v_{\delaynode}$ starts firing by round $\tau^*+2L$, iff $v \in V^+_{\sync}(p)$.
\end{proof}
We next show there exists a round in which the in-copies of $ V^+_{\sync}(p+1)$ begin to fire. 
\begin{claim} 
	For every $u \in V$ for its in-copy $u_{\inn}$ there is round $\tau_{u} \in [\tau^*+ c_1 \cdot L^2, \tau^*+c_2\cdot L^3+L]$ in which all its incoming neighbors in $V^+_{\sync}(p)$ fire, and the remaining neighbors do not fire.
\end{claim}
\begin{proof}
 The output neuron of $R_1$ fires in some round $\tau'\in [\tau^*+L+1, \tau^*+L^2+L]$, and therefore all neurons $v_{\outt}$ are inhibited by round $\tau^*+L^2+2L$. Recall that the delay module $D$ is implemented by $\DetTimer_{\async}$ with time parameter $2L^2$. Therefore the output neuron of $D$ fires in round $\tau_D \in [\tau^* + 2L^2 + 1 , \tau^*+10L^2]$, ensuring $D$ fires after $v_{\outt}$ was inhibited by $R_1$.
Recall that the reset module $R_2$ counts $L$ rounds after receiving a signal from $D$. This ensures that the inhibitory output of $R_2$ starts inhibiting $v_{\delaynode}$ after $v_{\outt}$ received the signal from $D$. 
By Claim~\ref{clm:delay-neuron} we conclude that when neuron $v_{\outt}$ receives the signal from the delay module $D$ in some round $\tau_{\outt}\in [\tau^*+c_1 \cdot L^2, \tau^*+ c_2 L^3]$, it fires iff $v \in V^+_{\sync}(p)$.
As a result, due to the self loops of the out-copies, $u_{\inn}$ has a round $\tau_{u} \in [\tau^*+c_1 \cdot L^2+1, \tau^*+c_2\cdot L^3+L]$ in which all its incoming neighbors in $V^+_{\sync}(p)$ fire.	
\end{proof}

Therefore, $u_{\inn}$ starts firing in round $\tau_u$ only if $u \in V^+_{\sync}(p+1)$ and it continues firing from round $\tau_{u}$ ahead in that phase due to the self loops of the out-copies of its neighbors.
Since the pulse generator fires to signal the end of phase $p+1$ in round $\tau^*+ c_3L^3 > \tau^*+c_2\cdot L^3+2L$, every neuron $v \in  V^+_{\sync}(p+1)$ fires in round $t(v,p+1)$ since both $g$ and $v_{\inn}$ fired previously (and other neurons are idle).
\end{proof}

\def\APPENDPFSYNC{
\paragraph{Why It Work? Proof Sketch.}
We assume that the firing states of the input layer is fixed throughout the simulation. 
For every round $p$, let $V^+_{\sync}(p) \subset V$ be the set of neurons that fire in round $p$ in the simulation of the synchronous network $\mathcal{N}_{\sync}$. 
Let $\tau^*$ be the round in which $g$ fires for the $p^{th}$ time, indicating the end of phase $p$. 
Recall that $t(v,p)=\tau^*+\ell(g,v)$ is the round in which neuron $v$ gets this signal from $g$. 
For every $v$ and phase $p$, we will maintain the invariant that $v$ fires in round $t(v,p)$ iff it fires in $\Pi_{\sync}$ in round $p$. 

Assuming that the invariant holds for phase $p$, we will show that it holds for phase $p+1$. 
Note that $t(v,p)\leq \tau^*+L$ for every $v \in V$, thus by round $\tau^*+L$ all the neurons in $V^+_{\sync}(p)$ have fired.  Since the delay neuron $v_{\delaynode}$ has a self-loop (with latency one), it starts firing in every round starting from round $\tau_{d} \in [\tau^*+2, \tau^*+2L]$. Our goal now is to copy the $p^{th}$-firing state of the $v$ neurons, as kept on being presented by $v_{\delaynode}$, to the out-copies $v_{\outt}$. This should be done only after the old state of $v_{\outt}$ (i.e., the firing state of $v$ in round $p-1$ in $\Pi_{\sync}$) gets deleted.

The first reset module $R_1$ receives its signal from $g$ in round $\tau_{r_1} \leq \tau^* +L$ and starts counting for $L$ rounds. Thus the output neuron of $R_1$ fires in some round $\tau'_{r_1} \in [\tau_{r_1}+L, \tau_{r_1}+L^2]$. This ensures that the inhibitory output of $R_1$ starts inhibiting $v_{\outt}$ (deleting its previous state) only \emph{after} all neurons in $V$ received the $PG$ signal (i.e., after all neurons in $V^+_{\sync}(p)$ have fired). This inhibitor inhibits $v_{\outt}$ for a single round, thus canceling its positive self-loop. 

At this point, we would like to copy the $p^{th}$ firing state of $v$ to its out-copy $v_{\outt}$ via the delay neuron $v_{\delaynode}$ that keeps on presenting this state. This should be done with care as we want this to happen only \emph{after} the reset of $v_{\outt}$ and not before (as otherwise, the reset will override the update state).
For this reason, the output copy $v_{\outt}$ is connected to the delay-module $D$, and will fire only when receiving the spike from both the output neuron of $D$ \emph{and} the delay-neuron $v_{\delaynode}$. 
The delay module $D$ is implemented by $\DetTimer_{\async}$ with time parameter $2L^2$. Therefore, the output neuron  of $D$ fires in round $\tau_{D} \in [\tau^*+2L^2, \tau^*+10L^3]$. Thus ensuring that it fires only \emph{after} all out-copies $v_{\outt}$ have been reset by the reset module $R_1$. Overall, we have that $v_{\outt}$ fires in some round $\tau_{\outt}\in [\tau^*+\Theta(L^2), \tau^*+\Theta(L^3)]$.

The reset module $R_2$ counts for $L$ rounds after receiving a signal from $D$. This ensures that the inhibitory output of $R_2$ inhibits $v_{\delaynode}$ only \emph{after} $v_{\outt}$ have already received the signal from $D$ in round $\tau_{\outt}$.
Our final goal is to end the $(p+1)^{th}$ phase with spikes from the neurons in $V^+_{\sync}(p+1)$. For that we need the $v_{\outt}$ neurons to present their current firing state (i.e., the firing state of $v$ in round $p$ of $\Pi_{\sync}$) for $\Omega(L)$ consecutive rounds. This is why we set the timer of the pulse-generator to be $\Theta(L^3)$. Due to their self-loops, in the last $2L$ rounds of phase $p+1$, the out-copy $v_{\outt}$ fires in each of these rounds for every $v \in V^+_{\sync}(p)$. As a result, for every in-copy $u_{\inn}$ there is a round in which all its incoming neighbors in $V^+_{\sync}(p)$ fire (and the remaining neighbors do not fire). This make $u_{\inn}$ starts firing only if $u$ fires in round $p+1$ in $\Pi_{\sync}$, i.e., if $u \in V^+_{\sync}(p+1)$. Since the out-copies keep on presenting the firing states of phase $p$ in each of the rounds, $u_{\inn}$ will continue firing for the last $\Omega(L)$ rounds of the phase. When the pulse-generator makes its $(p+1)^{th}$ pulse, the neurons $v \in  V^+_{\sync}(p+1)$ indeed fire as both $g$ and $v_{\inn}$ fire in the previous rounds. This completes the proof sketch of the construction. 

Finally we analyze the overhead in the number of neurons and in time. Using Awerbuch and Pelegs definition~\cite{AwerbuchP90b}, the time overhead is normalized by the longest ``round'' (edge latency in our case). Thus, we achieve a time overhead of $O(L^3)$ (the length of each phase). The total number of neurons we added is $O(L\log L)$ due to the global modules, and additional $O(n)$ neurons where $n$ in the number of auxiliary neurons in $V(\cN_{\sync})$. 
}

\vspace{-5pt}
\textbf{Acknowledgment:} We are grateful to Cameron Musco, Renan Gross and Eylon Yogev for various useful discussions.
\bibliographystyle{alpha} 
\bibliography{selection}

\newpage
\appendix

\section{Missing Details for the Introduction}\label{sec:append-intro} 
\paragraph{Proof of Observation~\ref{obs:lower-async}.}
We will show that implementing a simple NOT-gate in the asynchronous setting requires $\Omega(\log L)$ neurons. In the synchronous setting, one can easily implement a NOT-gate by connecting the input neuron to the output neuron with negative weight and setting the bias of the output to $0$.

Assume towards contradiction that there exists a deterministic network $\cN$ with $N =o(\log L)$ neurons, an input neuron $x$, and an output neuron $y$ that computes $y=NOT(x)$ within $T$ rounds. If $x$ fires in round $0$, the output $z$ should \emph{not} fire in any of the rounds $[0,T]$, and if $x$ does not fire, then there exists a round $\tau \in [0,T]$ in which $y$ fires.
We set the latencies on the edges of $\cN$ such that the outgoing edges from $x$ have latency $L$, and all other edges have latency $1$. 

Consider an execution $\Pi_{yes}$ where $x$ fires in rounds $[0,L+1]$, and an execution $\Pi_{no}$ where $x$ do not fire at all. The initial states of all other neurons are set to $0$ in both $\Pi_{yes}$ and $\Pi_{no}$.
By the correctness guarantee, during the execution $\Pi_{yes}$, the output neuron $y$ do not fire in rounds $[1,L+1]$, and during the execution $\Pi_{no}$ there exists a round $\tau \in [0,T]$ in which $z$ fires. 
Recall that the \emph{state} of the network in round $t$ is described by an $N$-length vector indicating the firing neurons in that round. Note that because the network contain $N =o(\log L)$ neurons, the network has at most $L/2$ distinct firing states.

Since the latency on all outgoing edges from $x$ is $L$, during rounds $[1,L]$ of execution $\Pi_{yes}$, the signal from $x$ does not reach any other neuron. Hence, the states of all neurons but $x$ during rounds $[1,L]$ of execution $\Pi_{no}$ are identical to those of execution $\Pi_{yes}$. In other words, except for the state of $x$, the two executions are indistinguishable over the first $L$ rounds.  By the correctness of $\Pi_{yes}$, we have that $y$ 
is idle during the first $L$ round, and therefore it is also idle in $\Pi_{no}$ during these rounds.

Since the network has at most $L/2$ distinct states, there must be a state $s$ that occurs at least twice during rounds $[1,L]$ in both executions $\Pi_{yes}$ and $\Pi_{no}$. In addition, in all the rounds between the two occurrences of $s$, the output $y$ does not fire (as $y$ is idle in the first $L$ rounds). Due to the memory-less property of the neurons, we conclude that the execution $\Pi_{no}$ is locked into a no-configuration in which $y$ will never fire, contradicting the correctness of the network.

\section{Missing Proofs for Det. Neural Timer}\label{sec:append-det} 
 \subsection{Complete Description of The $\DetTimer$ Network}\label{app:det}

\paragraph{Handling the General Case:} We begin by extending the $\DetTimer(t)$ network to handle the case where $x$ fired more than once within the execution.
\begin{itemize}
\item{\textbf{Case 1: $x$ fires several times within a span of $t$ rounds.}} We introduce an additional reset (inhibitory) neuron $r$ that receives input from $x$ with weight $w(x,r)=1$, has outgoing edges to all neurons except $a_{1,2}$ and $y$ with negative weight of $-2$, and threshold value $b(r)=1$.
	
\item{\textbf{Case 2: $x$ fires again just one round before $a_{\log \hat{t},2}$ fires.}} To process this new spike, we introduce a control neuron $c$ that receives input from $x$ with weight $w(c,x)=1$ and threshold $b(c)=1$ and fires one round after $x$. The control neuron $c$ has outgoing edges to $y$ and $a_{1,2}$ with weights $w(c,y)=w(c, a_{1,2})=3$. Therefore even if $a_{\log \hat{t},2}$ fires one round after $x$, the control neuron will cancel the inhibition on the output $y$ and on $a_{1,2}$ and the timer will continue to fire. 
\end{itemize}
Figure~\ref{fig:dettime} illustrates the structure of the network.

\begin{figure} 
	\hspace*{\fill}
	\includegraphics[scale=0.3]{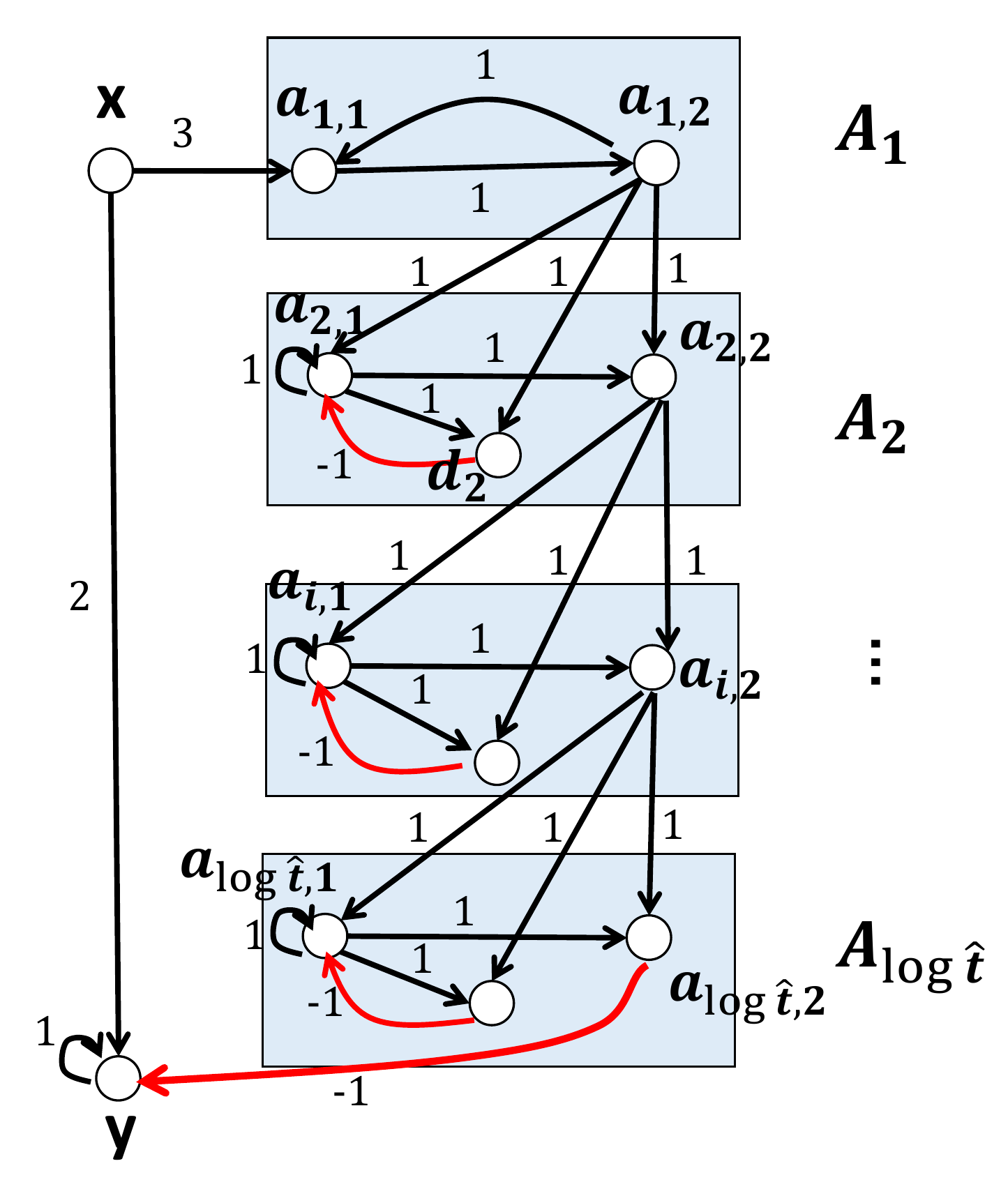} \hfill
	\includegraphics[scale=0.3]{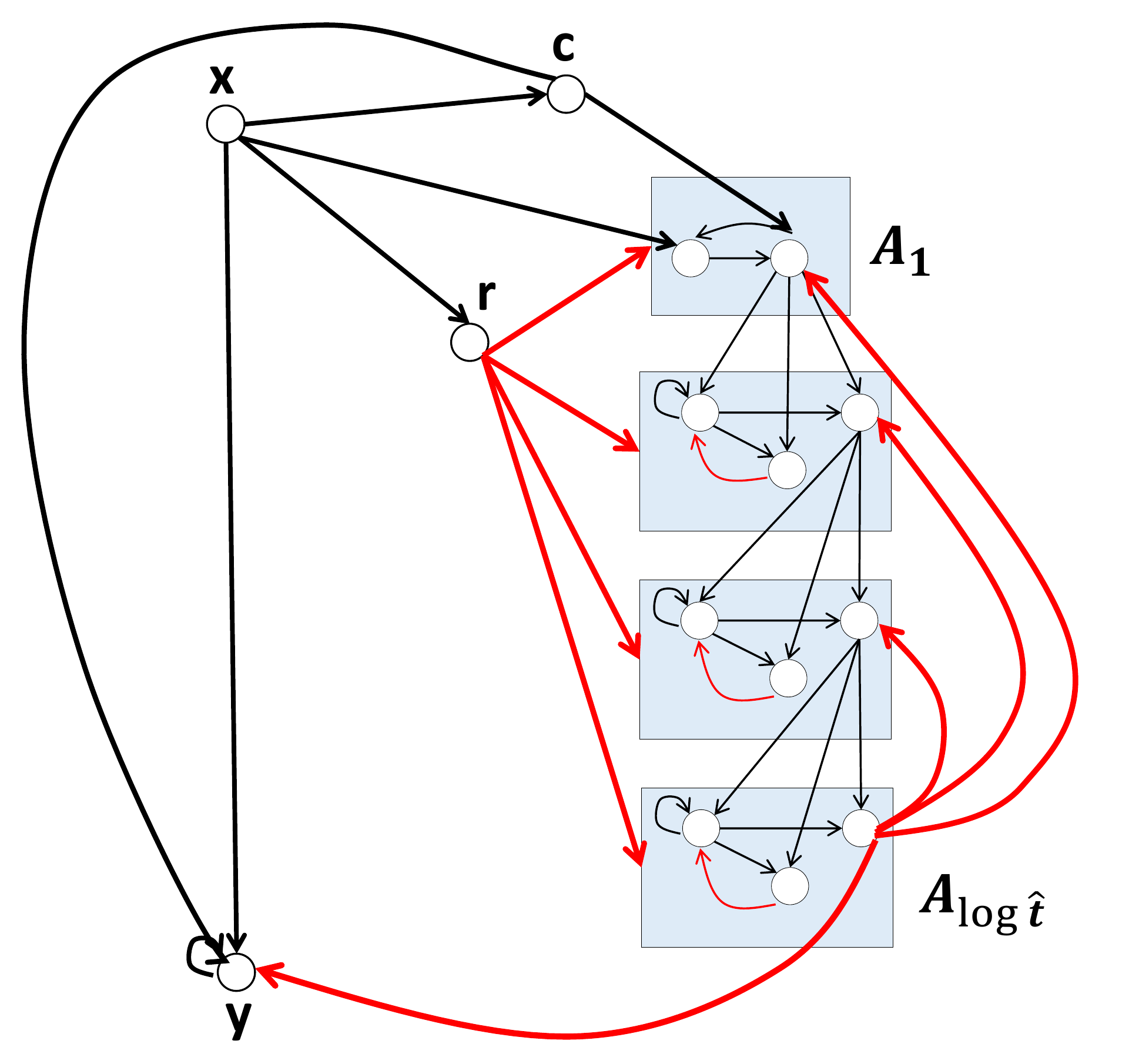}
	\hspace*{\fill}
	\centering
	\caption{Illustration of the $\DetTimer$ network. \small{Left: The simplified network for the case that $x$ fired once.  The neurons $y$, $a_{1,2}$ and the set of neurons $\{a_{1,1}, \ldots ,a_{\log \hat{t},1}\}$ have threshold $1$. For $i \geq 2$ the threshold of $a_{i,2}$ and $d_i$ is $2$. Right: A complete network description for the general case, where the input can fire several times during the execution. The reset neuron $r$ resets the timer in case $x$ fires several times. The control neuron $c$ takes care of the special extreme case where $x$ fires again one round \emph{before} the last counting neuron $a_{\log \hat{t},2}$ fires. \label{fig:dettime}}}
\end{figure}

We next use Claim~\ref{clm:layers} in order to prove the the first part of Thm.~\ref{lem:upper-bound-det}.

\subsection{Complete Proof of Thm.~\ref{lem:upper-bound-det}(1)} \label{sec:det-proof}
\begin{proof}
We start by considering the case where $x$ fires once in round $t'$. 
If $x$ fired in round $t'$, due to the self loop of $y$, starting from round $t'+1$, the output keeps firing as long as $a_{\log \hat{t},2}$ did not fire. By Claim~\ref{clm:layers}, $a_{\log \hat{t},2}$ fires in round $t' + \hat{t}+\log \hat{t}-1 =  t' + t -1 $, and therefore $y$ will be inhibited in round $ t' + t$. Note that $a_{\log \hat{t},2}$ also inhibits all other auxiliary neurons, and therefore as long as $x$ will not fire again, $y$ will also not fire. 
Next we consider the case where $x$ also fired in round $t'' \geq t'+1$.
\begin{itemize}
\item{\textbf{Case 1: $t'' \geq t' + t$}.} Because in round $t' + t - 1$ the neuron $a_{\log \hat{t},2}$ inhibits all counting neurons in the network, starting round $t' + t$ no counting neuron fires until $x$ fires again in round $t''$. Thus, after $x$ fires in round $t''$ the network behaves the same as after the first firing event.
\item{\textbf{Case 2: $t'' \leq t' + t - 3$}}. In round $t''+1 \leq t' + t - 2$ the reset neuron $r$ inhibits all counting neurons except for $a_{1,2}$. Hence, in round $t''+2$ only $y$ and $a_{1,2}$ fire, and the neural timer continues to count for additional $t-2$ rounds.
\item{\textbf{Case 3: $t''=t' + t-1$}.} The neuron $a_{\log \hat{t},2}$ fires on the same round as $x$. Since the weights on the edges from $x$ to $y$ and $a_{1,1}$ are greater than the weight of the inhibition from $a_{\log \hat{t},2}$, the timer continues to fire based on the last firing event of $x$.
\item{\textbf{Case 4: $t'' = t' + t - 2$}.}  In this case $x$ fires in round $t''$ and in the next round, $a_{\log \hat{t},2}$ fires and inhibits the output $y$ (at the same round that the reset neuron $r$ fires). Recall that in round $t''+1$ the control neuron $c$ also fires. Hence, in round $t''+2$ neuron $c$ excites $y$ and $a_{1,2}$ canceling the inhibition of $a_{\log \hat{t},2}$. 
\end{itemize} 
\end{proof}

\subsection{Useful Modifications of Deterministic Timers}\label{sec:mod-timer}
We show a slightly modified variant of neural timer denoted by $\DetTimer^*$ which receives as input an additional set of $\log t$ neurons that encode the desired duration of the timer.

\paragraph{(1) Time Parameter as a Soft-Wired Input.} 
The $\DetTimer$ construction is modified to receiving a time parameter $t' \leq t$ as a (soft) input to the network.
That is, we assume that $t$ is the upper limit on the time parameter. The same network can be used as a timer for any $t' \leq t$ rounds, and this $t'$ can be given as an input to the network. In such a case, once the input neuron $x$ fires, the output neuron $y$ will fire for the next $t'$ consecutive rounds. 
The time parameter $t'$ is given in its binary form using $\log t$ input neurons denoted as $z_1 \dots z_{\log t}$. We denote this network as $\DetTimer^*(t)$.
The idea is that given time parameter $t'$, we want to use only $\log (t'')$ layers out of the $\log t$, where $t''=t' +\log (t')$ (we use $t''$ due to the $\log (t'')$ delay in the update of the timer). The modifications are as follows.

\begin{enumerate}
\item  The time input neurons are set to be inhibitors.
\item  The intermediate layer of neurons $c_1 \ldots c_{\log t''}$ determine how many layers we should use. Each $c_i$ has negative edges from $z_1, \ldots , z_{\log t''}$ with weights $w(c_i,z_j)=-2^{j-1}$, and threshold $b(c_i)=-i-1-2^{i-1}$. Hence $c_i$ fires iff $i-1+2^{i-1} \geq \dec(\bar{z})=t'$. 

\item We introduce $\log t''$ inhibitors $r_1, \ldots r_{\log t''}$ in order to inhibit the output $y$ after we count to $t'$ and reached layer $t''$. Each $r_i$ has incoming edges from $c_i$ and $a_{i,1}$, and fires as an AND gate. Hence, each $r_i$ fires only when the timer count reach $2^{i-1}+i-1$ and $i-1 + 2^{i-1} \geq t'$.
\item The output neuron $y$ receives negative incoming edges from the neurons $r_1 \ldots r_{\log t''}$ with weight $w(r_i,y)=-1$, and stops firing if at least one neuron $r_i$ fired in the previous round. 
\item Every neuron $r_i$ also has negative outgoing edges to all counting neurons $a_{j,k}  \ \  k \in \{1,2\}, j=1 \ldots \log t$ with weight $w(r_i, a_{j,k})=-2$ in order to reset the timer when we finish counting to $t'$. 
\end{enumerate}
See Figure~\ref{fig:det-param} for an illustration of $\DetTimer^*(t)$ network.
\begin{figure} 
\begin{center}
\includegraphics[scale=0.35]{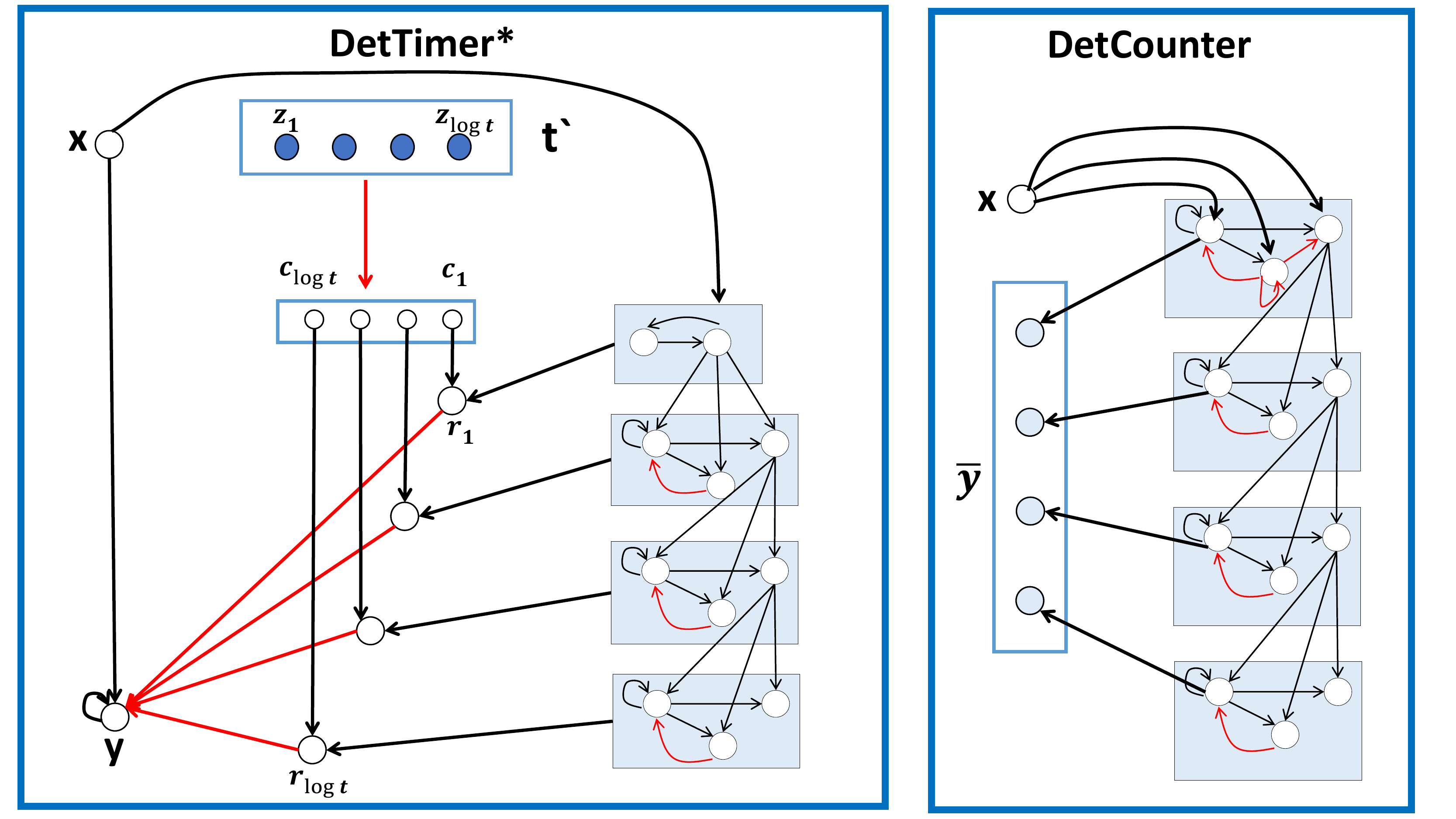} 
\caption{Left: Det. neural timer $\DetTimer^*$ with a soft-wired time parameter. The input neurons $z_1, \ldots, z_{\log t}$ encode the time parameter $t'$. The intermediate neurons $c_1, \ldots ,c_{\log t}$ control how many layers are used depending on the time parameter $t'$. Once the timer reaches layer $i=\Theta(\log (t'))$ for which $c_i$ fires, the inhibitor $r_i$ inhibits the output $y$ and the counting terminates. Right: neural counter $\DetCounter$, the output neurons $\bar{y}$ encode the number of times $x$ fired in a time window of $t$ rounds. \label{fig:det-param}}
\end{center}
\end{figure}

\paragraph{(2) Extension to Neural Counting.}
We next show a modification of the timer into a neural counter network $\DetCounter$ that instead of counting the number of rounds, counts the number of input spikes in a time interval of $t$ rounds. This network also uses $O(\log t)$ auxiliary neurons. To improve upon this bound, we resort to approximation and in Appendix \ref{sec:counting}, we combine the $\DetCounter$ network with the streaming algorithm of \cite{Flajolet85} to provide an approximate counting network with $O(\log\log t+\log(1/\delta))$ neurons where $\delta$ is the error parameter.
We next describe the required adaptation for constructing the network described in Lemma~\ref{lem:det-counter}.
 
The $\DetCounter$ with parameter $t$ contains $\log t$ layers, all layers $i \geq 2$ are the same as in $\DetTimer$ and only the first layer is slightly modified. 
The first counting neuron $a_{1,1}$ has a positive incoming edge from $x$ with weight $w(x,a_{1,1})=4$, and a self loop with weight $w(a_{1,1},a_{1,1})=1$. In addition $a_{1,1}$ has a negative edge from the inhibitor $d_1$ with weight $w(d_1,a_{1,1})=-1$, and threshold $b(a_{1,1})=1$. The second counting neuron $a_{1,2}$ has positive edges from $x$ and $a_{1,1}$ with weights $w(x,a_{1,2})=w(a_{1,1},a_{1,2})=1$, a negative edge from $d_1$ with weight $w(d_1,a_{1,2})=-2$ and threshold $b(a_{1,2})=2$. The reset neuron $d_1$ is an inhibitor copy of $a_{1,2}$ and therefore also has positive edges from $x$ and $a_{1,1}$ with weights $w(x,d_1)=w(a_{1,1},d_1)=1$, a negative self loop with  weight $w(d_1,d_1)=-2$ and threshold $b(d_1)=2$. 
We then connect the counting neurons $a_{1,1}, \cdots a_{\log t,1}$ to the output vector directly, where $y_i$ has an incoming edge from $a_{i,1}$ with weight $w(a_{i,1},y_i)=1$ and threshold $b(y_i)=1$. 
Figure~\ref{fig:det-param} demonstrate the $\DetCounter(t)$ network.

Next we show that once the counter is updated, the number of times that $x$ fired is represented as a binary number where the counting neuron $a_{i,1}$ represents the $i^{th}$ bit in the binary representation ($a_{1,1}$ is the least significant bit).
We note that if the last firing of $x$ occurs is in round $\tau$ then after at most $\log c + 1$ rounds the counter is updated with the new value, where $c$ is the value of the counter before round $\tau$.
 We start by showing the following claim concerning the first layer.
\begin{claim} \label{clm:first-layer}
	If $x$ fired in round $\tau$, neurons $d_1$ and $a_{1,2}$ fire in round $\tau+1$ iff $x$ fired an even number of times by round $\tau$.
\end{claim}
\begin{proof}
	By induction on the number of times $x$ fired, denoted as $n$. Since $d_1$ and $a_{1,2}$ have identical potential functions it is sufficient to prove the claim for the neuron $d_1$. For $n=1$, if $x$ fired once in round $\tau$, then $a_{1,1}$ fires for the first time in round $\tau+1$, and since $d_1$ fires only if $a_{1,1}$ fired in the previous round, in round $\tau+1$ both neuron $d_1$ and $a_{1,2}$ are idle. For $n=2$, since $x$ fired for the first time in some round $\tau'\leq \tau-1$, starting round $\tau'+1$ neuron $a_{1,1}$ fires on every round until $d_1$ fires. Hence, in round $\tau+1$ the neuron $d_1$ receives spikes from both $x$ and $a_{1,1}$ and therefore fires. Assume the claim holds for every $k \leq n-1$ and we will show correctness for $n$. Denote the round in which $x$ fired for the $(n-1)^{th}$ time by $\tau'\leq  \tau-1$.
\begin{itemize}
\item (Case $1$: $n$ is even.) Since $n-1$ is odd, by the induction assumption $d_1$ did not fire in round $\tau'+1$. Hence $a_{1,1}$ is not inhibited until round $\tau+1$, and due to the self loop $a_{1,1}$ also fires in round $\tau$. Therefore $d_1$ and $a_{1,2}$ fire in round $\tau+1$. 
\item (Case $2$: $n$ is odd.) If $\tau'=\tau-1$, by the induction assumption $d_1$ fires in round $\tau'+1=\tau$, and due to the negative edges from $d_1$, both $d_1$ and $a_{1,2}$ are idle in round $\tau+1$. Otherwise, $\tau'\leq\tau-2$. By the induction assumption, $d_1$ fires in round $\tau'+1$. Since $x$ did not fire in round $\tau'+1$ (as it fires again only in round $\tau$), in round $\tau'+2\leq \tau$ the neuron $a_{1,1}$ is inhibited by $d_1$ and therefore in round $\tau$ the neurons $d_1$ and $a_{1,2}$ receives a signal only from $x$ and does not fire. 
\end{itemize}
\end{proof}
Next, we show that if $x$ fired in round $\tau$ for the last time, for each layer $i \in [1, \log n]$, neuron $a_{i,2}$ fires in round $\tau+i$ only if $x$ fired $\ell \cdot 2^{i-1}$ times by round $\tau$ for some integer $\ell \geq 1$.

\begin{claim} \label{obs:all-layers}
For every layer $i \in [2,\log t]$ if $a_{i-1,2}$ fired in round $\tau$ for the $n^{th}$ time, the neurons $d_i$ and $a_{i,2}$ fire in round $\tau+1$ iff $n$ is even.
\end{claim}
\begin{proof}
By induction on $n$. For $n=1$, one round after the first time neuron $a_{i-1,2}$ fires, the neuron $a_{i,1}$ fires for the first time, and therefore $a_{i,2}$, $d_i$ do not fire. For $n=2$, the second time $a_{i-1,2}$ fires, due to the self loop on $a_{i,1}$, it fires as well and therefore after one round $a_{i,2}$ and $d_i$ fire. Assume that $a_{i-1,2}$ fired in round $\tau'$ for the $(n-1)^{th}$ time.
If $n$ is even, then by the induction assumption $d_i$ does not fire in round $\tau'+1 \leq \tau$. Hence, due to the self loop of $a_{i,1}$, in round $\tau$ also $a_{i,1}$ fires and therefore $d_i$ and $a_{i,2}$ fire in round $\tau+1$. If $n$ is odd, by the induction assumption $d_i$ fires in round $\tau'+1$.
By Claim~\ref{clm:first-layer} there is at least one round distance between every two firing events of $a_{1,2}$. Thus, there is at least one round distance between every two firing events of $a_{i-1,2}$, and therefore $\tau \geq \tau'+2$. Hence, because $a_{i,1}$ was inhibited by $d_i$ in round $\tau'+1<\tau$, it is idle in round $\tau$ and the neurons $d_i$ and $a_{i,2}$ do not fire in round $\tau+1$. 
\end{proof}
\begin{corollary}
If $x$ fired for the $n^{th}$ time in round $\tau$, for every layer $i \in [1,\log t]$ the neurons $d_i$ and $a_{i,2}$ fire in round $\tau+i$ iff $(n \bmod 2^i) =0$.
\end{corollary}
\begin{proof}
By induction on $i$. The base cases for $i=1$ follows from Claim~\ref{clm:first-layer}. Assume that the claim holds for layer $i$ and we will show it also holds for layer $i+1$. If $(n \bmod 2^i) = 0$, then $n = q\cdot 2 \cdot 2^{i-1}$ for some integer $q$. Therefore by the induction assumption, $a_{i,2}$ fires in round $\tau+i$, and moreover it fired an even number of times by that round. Hence, by Claim~\ref{obs:all-layers} the neurons $d_{i+1}$ and $a_{i+1,2}$ fire in round $\tau+i+1$.
Otherwise, if $(n \bmod 2^i) \neq 0$, by the induction assumption $a_{i,2}$ does not fire in round $\tau+i$ and therefore $d_{i+1}$ and $a_{i+1,2}$ do not fire in round $\tau+i+1$. If $(n \bmod 2^i) = 0$ but $(n \bmod 2^{i+1}) \neq 0$, then by the induction assumption $a_{i,2}$ fired an odd number of times by round $\tau+i$ and by Claim~\ref{obs:all-layers} neurons $d_i$ and $a_{i,2}$ do not fire in round $\tau+i+1$.
\end{proof}
The first counting neuron $a_{i,1}$ fires one round after $a_{i-1,2}$ fires, and as long as $d_i$ and $a_{i,2}$ did not fire. Hence, we can conclude that if $x$ fired for the last time in round $\tau$, by round $\tau+ \log r_\tau+1$, the neurons $a_{1,1}, \ldots , a_{\log t,1}$ hold a binary representation of the number of times $r_{\tau}$ that $x$ fired by round $\tau$.

\section{Approximate Counting}\label{sec:counting}
In this section, we provide improved constructions for neural counters by allowing approximation and randomness. 
Our construction is inspired by the \emph{approximate counting} algorithm of Morris as presented in \cite{Morris78, Flajolet85} for the setting of dynamic streaming.  The idea is to implement a counter which holds the logarithm of the number of spikes with respect to base $\alpha=1 + \Theta(\delta)$. 
The approximate neural counter problem is defined as follows.
\begin{definition}[(Approximate) Neural Counter]
Given a time parameter $t$ and an error probability $\delta$, an approximate neural counter has an input neuron $x$, a collection of $\log t$ output neurons represented by a vector $\bar{y}$, and additional auxiliary neurons. The network satisfies that in a time window of $t$ rounds, in every given round, the output $\bar{y}$ encodes a constant approximation of the number of times $x$ fired up to that round, with probability at least $1 -\delta$.
\end{definition}
Throughout, we assume that $1/\delta < t $. For smaller values of $\delta$, it is preferable to use the deterministic network construction of $\DetCounter$ with $O(\log t)$ neurons described in Lemma~\ref{lem:det-counter}. 
For the sake of simplicity, we first describe the construction under the following promises:
\begin{itemize}
\item{(S1)}
The firing events of $x$ are sufficiently spaced in time, that is there are $\Omega(\log t)$ rounds between two consecutive firing events. 
\item{(S2)}
The state of $\bar{y}$ encodes the right approximation in every round $\tau$ such that the last firing of $x$ occurred before round $\tau - \log r_\tau$ where $r_\tau$ is the number of $x$'s spikes (firing events) up to round $\tau$.
\end{itemize}

\paragraph{High Level Description.} 
The network $\ApproxCounter(t,\delta)$ consists of two parts, one for handling small number of spikes by the input $x$ and one for handling the large counts. The first part that handle the small number of spikes is deterministic. 
Specifically, as long as the number of spikes by $x$ is smaller than $s=\Theta(1/\delta^2)$, we count them using the exact neural counter network (presented in Appendix~\ref{sec:mod-timer}), using $O(\log 1/\delta)$ neurons. We call this module 
\emph{Small Counter} ($SC$) and it is implemented by the $\DetCounter$ network with time parameter $\Theta(1/\delta^2)$. 

To handle the large number of spikes, we introduce the \emph{Approximate Counter} ($AC$) implemented by $\DetCounter$ with time parameter $\log_{\alpha} t$. The $AC$ module approximates the logarithm of the number of rounds $x$ fired with respect to base  $\alpha = 1 + \Theta(\delta)$. This module is randomized, and provides a good estimate for the spikes count given that it is sufficiently large.
The idea is to update the $AC$ module (by adding +1) upon every firing event of $x$ with probability $\frac{1}{1+\alpha^c}$ where $c$ is the current value stored in the counter
To do so, we have a spiking neuron $a^*$ that has incoming edges from the output of the $AC$ module, and fires with the desired probability. The reason we use probability $\frac{1}{1+\alpha^c}$ instead of $\frac{1}{\alpha^c}$ as suggested in Morris algorithm, is due to the sigmoid probability function of spiking neurons (see Eq. (\ref{eq:potentialOut})).
Once the count is large enough (more than $s$), we start using the $AC$ module. This is done by introducing an indicator neuron $v_I$, indicating that the small-counter is full. 
The neuron $v_I$ starts firing after $SC$ is full (finished the count), and keeps on firing due to a self loop.

The input to $AC$, denoted as $x_{ac}$ computes an AND of the input $x$, the spiking neuron $a^*$ and the indicator neuron $v_I$.
In addition, $v_I$ initiates a reset of the small counter $SC$ to make sure that the output $\bar{y}$ receives only information from the large-count module $AC$. 
Figure~\ref{fig:approx_counting} provides a schematic description of the construction.
\begin{figure} 
\begin{center}
\includegraphics[scale=0.3]{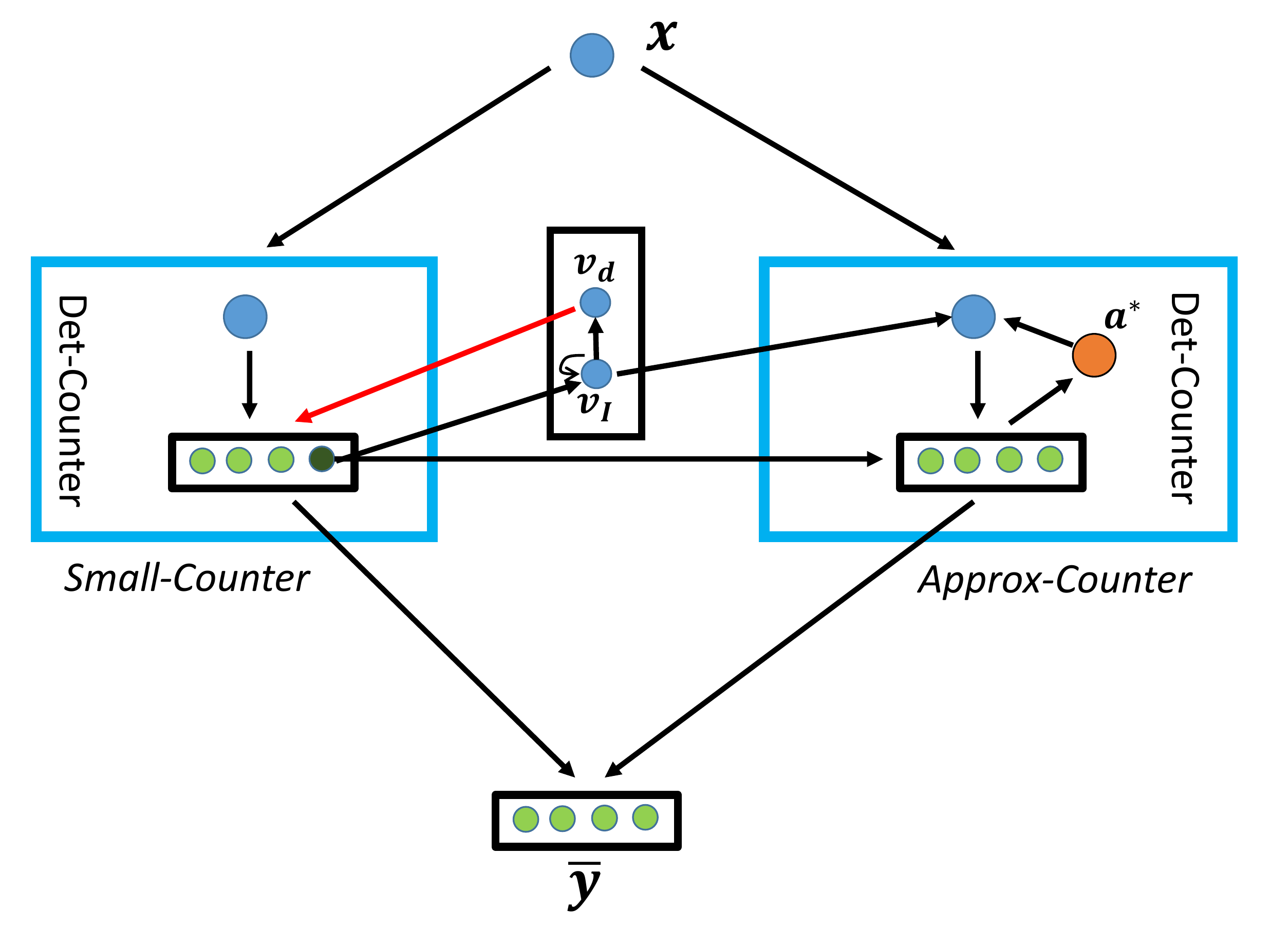}
\caption{Schematic description of the network $\ApproxCounter$. In each module only the input and output layers are shown. The Small-Counter module $SC$ is responsible for counting up to $\Theta(\log 1/\delta)$ spikes, and is implemented by the $\DetCounter$ module with time parameter $\Theta(1/\delta)$. For handling large counts, we use the Approx-Counter  $AC$ module implemented by the $\DetCounter$ module with time parameter $\Theta(\log t/\delta)$. The Approx-Counter module simulates Morris' algorithm and maintain an estimate for the logarithm of the spikes count. The neurons $v_I$ and $v_r$ switch between the two stages (small and large counts) during the execution. \label{fig:approx_counting}}
\end{center}	
\end{figure} 

\paragraph{Detailed Description.}
Let $r_n$ be the number of times $x$ fired in the first $n$ rounds, and let $\alpha= 1+\Theta(\delta)$ be the base of the counting in the approximate counting module. 
\begin{itemize}
\item \textbf{Handling Small Counts.} 
The module \emph{Small-Counter} ($SC$) is implemented by the $\DetCounter$ module with time parameter $s$ and input from $x$, where $s = \frac{1}{\delta(\alpha-1)}$. Since $\alpha = 1+ \Theta(\delta)$, it holds that $s = \Theta(1/\delta^2)$. 
In addition, we introduce an excitatory \emph{indicator} neuron $v_I$ that has an incoming edge from the last layer of $SC$ (i.e. neuron $a_{\log s,2}$) as well as a self loop, each with weight $1$ and threshold $b(v_I)=1$. The indicator neuron $v_I$ has an outgoing edge to an inhibitory \emph{reset} neuron $v_r$ with weight $w(v_I, v_r)=1$, which is connected to all neurons in $SC$ with negative weight $-5$. The reset neuron $v_r$ also has an incoming edge from $a_{\log s,2}$ with weight $1$ and threshold $b(v_r)=1$.  As a result, one round after $SC$ reaches the value $s$, it is inhibited.

\item \textbf{Handling Large Counts.} The \emph{Approximate-Counter} ($AC$) is implemented by a $\DetCounter$ module with time parameter $\log_{\alpha} t$, and its input neuron is denoted by $x_{ac}$. Denote by $\ell = \log_2 \log_{\alpha} t$ the number of layers in the $AC$ module, and for every $1 \leq  i \leq \ell$, denote the counting neuron $a_{i,1}$ by $c_i$.
To initialize the counter we connect the last output neuron of $SC$ to the counting neurons $c_{i_1} \ldots c_{i_k}$ in $AC$ which correspond to the binary representation of $\log_{\alpha}(1/\delta+1)$ with weights $5$.
We introduce a probabilistic spiking neuron $a^*$ that is used to increase the counter with the desired probability. In order for $a^*$ to receive negative weights from $AC$, we connect each counting neuron $c_i$ to an inhibitor copy $c_{i,2}$ with weight $w(c_i,c_{i,2})=1$ and threshold $1$. We then connect the inhibitors $c_{1,2}, \dots ,c_{\ell,2}$ to $a^*$ with weights $w(c_{i,2},a^*)=-2^{i-1}\cdot  \ln \alpha$, and set $b(a^*)=0$.
Hence, $a^*$ fires in round $\tau$ with probability $\frac{1}{1+\alpha^c}$, where $c$ is the value of $AC$ in round $\tau-1$.
Finally, the input neuron $x_{ac}$ has incoming edges from $a^*$, $x$ and $v_I$ each with weight $1$ and threshold $b(x_{ac})=3$. As a result, $x_{ac}$ fires only if $v_I$, $x$ and $a^*$ fired in the previous round. 

\item \textbf{The Output Neurons.} The counter modules $SC$ and $AC$ are connected to the output vector $\bar{y}$ as follows. Each $y_i$ has incoming edges from neurons $c_1, \ldots ,c_{\ell}$ with weight $w(c_i,y)= \log \alpha \cdot 2^{i-1}$, and threshold $b(y_i)=i + \log(\alpha-1)$. In addition, each output neuron $y_i$ has an incoming edge from the $i^{th}$ output of $SC$ with weight $b(y_i)$.  
Hence, $y_i$ fires in round $\tau$ if either
$\log \alpha \cdot (\sum_{j=1}^{_{\ell}}c_j\cdot 2^{j-1})-\log(\alpha-1) \leq i$, or the  $i^{th}$ output of $SC$ fired in the previous round. 
\end{itemize}

\paragraph{Size Complexity.}
All neurons except the spiking neuron $a^*$ are threshold gates. 
Recall that $\alpha= 1+\Theta(\delta)$. Hence the size of the counter $AC$ is $O(\log_2 \log_{\alpha} t) = O(\log\log t+ \log(1/\delta))$. Since the size of the counter $SC$ is $O(\log 1/\delta)$, overall we have $O(\log\log t+ \log(1/\delta))$ auxiliary neurons.

\paragraph{Analysis (under the simplifying assumptions).}
We first show the correctness of the $\ApproxCounter$  construction under the two promises. 
At the end of the section we will show correctness for the general case as well. 
Let $r_{\tau}$ be the number of times $x$ fired up to round $\tau$. If $r_{\tau} \leq s$ the correctness of $\ApproxCounter$ follows from the correctness of the $\DetCounter$ construction (see Lemma~\ref{lem:det-counter}). 
From now on, we assume $r_{\tau} \geq s+1$. 
Let $z_n$ be a random variable holding the value of $AC$ after $x$ fired $n$ times (i.e when $r_\tau= n$). 
We start with bounding the expectation of $\alpha^{z_n}$.

\begin{claim} \label{clm:mean}
$\mathbb{E}[\alpha^{z_n}] \in [n(\alpha-1)(1-\delta)+1, n(\alpha-1)+1]$. 
\end{claim}
\begin{proof}
The $AC$ counter starts to operate after $x$ fired $s = \frac{1}{\delta(\alpha-1)}$ spikes, and we initiate the counter with value $c =  \log_{\alpha}(1/\delta+1)$. Hence, for $n=s$ we get  $\alpha^{z_n} = n(\alpha-1)+1$ and the claim holds.
For $n \geq s+1$ we get
\begin{eqnarray}
\mathbb{E}[\alpha^{z_n}]&=&\sum_{j=c}^{n-1} \mathbb{E}[\alpha^{z_n} ~\mid~ z_{n-1} = j] \cdot \Pr[z_{n-1} = j] \nonumber
\\&=&  \sum_{j=c}^{n-1} \Pr[z_{n-1} = j]\cdot (\alpha^{j+1}\cdot \frac{1}{\alpha^j+1}+ \alpha^j \cdot (1 - \frac{1}{\alpha^j+1})) \nonumber
\\&=& \mathbb{E}[\alpha^{z_{n-1}}]+(\alpha-1)\cdot\sum_{j=c}^{n-1} \Pr[z_{n-1} = j]\cdot (\frac{\alpha^j}{1+\alpha^j})~. \label{eq:mean}
\end{eqnarray}
Note that for $j \geq c$, it holds that $1 > \frac{\alpha^j}{1+\alpha^j} > 1-\delta$. Therefore $$\sum_{j=c}^{n-1} \Pr[z_{n-1} = j]\cdot (\frac{\alpha^j}{1+\alpha^j}) \in [1-\delta,1]~.$$ By combining this with Eq.~\eqref{eq:mean} we conclude that $\mathbb{E}[\alpha^{z_n}] \in [n(\alpha-1)(1-\delta), n(\alpha-1)+1]$.
\end{proof}
\begin{claim} \label{clm:cheb}
$\Pr[|\alpha^{z_n}-\mu|> 1/2\cdot \mu] \leq \delta$, where $\mu=\mathbb{E}[\alpha^{z_n}]$.
\end{claim}
\begin{proof}
We will use Chebyshev's inequality, and start by computing $\mathbb{E}[\alpha^{2z_n}]$ in order to bound the variance of $\alpha^{z_n}$.
\begin{eqnarray}
\mathbb{E}[\alpha^{2z_n}] &=& \sum_{j=c}^{n-1} \mathbb{E}[\alpha^{2z_n} ~\mid~ z_{n-1} = j] \cdot \Pr[z_{n-1} = j] \nonumber
\\&=& 
\sum_{j=c}^{n-1} \Pr[z_{n-1}= j]\cdot \left(\alpha^{2j+2}\cdot \frac{1}{\alpha^j+1}+ \alpha^{2j} \cdot (1 - \frac{1}{\alpha^j+1})\right) \nonumber
\\&=& \mathbb{E}[\alpha^{2z_{n-1}}]+\sum_{j=c}^{n-1} \Pr[z_{n-1} = j]\cdot (\frac{\alpha^{2j}(\alpha^{2}-1)}{\alpha^j+1}) \nonumber
\leq
\mathbb{E}[\alpha^{2z_{n-1}}] + (\alpha^{2}-1)\mathbb{E}[\alpha^{z_{n-1}}] 
\\&\leq& 	\mathbb{E}[\alpha^{2z_{n-1}}] + (\alpha^{2}-1)\cdot ((n-1)(\alpha-1)+1)~, \label{eq:var1}
\end{eqnarray}
where Ineq. \eqref{eq:var1} is due to Claim~\ref{clm:mean}.
For $n=s$, it holds that $$\mathbb{E}[\alpha^{2z_{s}}] = s^2(\alpha-1)^2+2s(\alpha-1)+1 \leq (\alpha+1)(\alpha-1)\sum_{i=1}^{s} i + (\alpha-1)(\alpha+1)s~,$$
and combined with Eq. \eqref{eq:var1} we get
$$\mathbb{E}[\alpha^{2z_n}] \leq \frac{1}{2}\left( n(3\alpha^2-\alpha^3+\alpha-3)+n^2(\alpha+1)(\alpha-1)^2\right)~.$$

Therefore the variance is bounded by 
\begin{eqnarray}	
	Var[\alpha^{z_n}] &=& \mathbb{E}[\alpha^{2z_n}]- (\mathbb{E}[\alpha^{z_n}])^2 \nonumber
	\\&\leq&
	 \frac{1}{2}n^2(\alpha-1)^2\left((\alpha-1)+\delta^2\right)+
	n\left((\alpha-1)(2\alpha+1-a^2)+2\alpha\delta\right)~.\nonumber
\end{eqnarray} 
Using Chebyshev's inequality and Claim~\ref{clm:mean} we can now conclude the following:
\begin{eqnarray}
\Pr[|\alpha^{z_n}-\mu|\geq 1/2\cdot \mu] &\leq& \frac{Var[\alpha^{z_n}]}{((1/2)\cdot \mu)^2} \leq 
\frac{4Var[\alpha^{z_n}]}{n^2(\alpha-1)^2(1-\delta)^2} \nonumber
\\&\leq& 4\left((\alpha-1)+\delta^2\right)  + \frac{8n\left(\alpha-1+2\delta\right)}{n^2(\alpha-1)^2(1-\delta)^2}~, \label{eq:chebi}
\end{eqnarray} 
since we assume $n \geq s = 1/\delta(\alpha-1)$ it holds that $n \leq n^2(\alpha-1)\delta$. As a result, by Eq.~\eqref{eq:chebi} we get: 
\begin{eqnarray}	
\Pr[|\alpha^{z_n}-\mu|\geq 1/2\cdot \mu] &\leq& 
 4\left((\alpha-1)+\delta^2\right)+ 10\delta\left( 1 + 2\delta/(\alpha-1)  \right)~. \nonumber
\end{eqnarray} 
Since $\alpha = 1 + \Theta(\delta)$, we have that $Var[\alpha^{z_n}] \leq \Theta(\delta)$.
We can use $\delta' = \Theta(\delta)$ in our construction and set parameter $\alpha$ accordingly in order to achieve $$\Pr[|\alpha^{z_n}-\mu|\geq 1/2\cdot \mu] \leq \delta~.$$
\end{proof}

Combining Claim~\ref{clm:mean} and Claim~\ref{clm:cheb} we conclude that  $\alpha^{z_n} \in [n(\alpha-1)/4, 2n(\alpha-1)]$ with probability at least $1- \delta$. Let $S= \log \alpha \cdot z_n- \log(\alpha-1)$. 	
Thus, $S \in [\log(n/4),\log(2n)]$. Recall that after $SC$ gets reset, each $y_i$ fires only if $\log \alpha \cdot z_n - \log(\alpha-1) \leq i$.
As a result, the value of the output $\bar{y}$ is given by $$\dec(\bar{y})=\sum_{i=1}^{S} 2^i = 2^{S+1}-2 \in [n/2-2, 4n-1] ~,$$ which is a constant approximation of $n$ as desired.  

\paragraph{Adaptation to the General Case.} We now explain the modifications needed to handle the general case without the two simplifying assumptions. In order to fire with the correct probability without the spacing guarantee, every time we increase $AC$, we wait until its value gets updated before we attempt to increase it again. In order for the output $\bar{y}$ to output the correct value also during the update of the counter $AC$, we introduce an intermediate layer of neurons $c''_1, \ldots , c''_{\ell}$ that will hold the previous state of $AC$ during the update.

\begin{itemize}
\item{\textbf{Removing Assumption (S1):}} In the $\DetCounter$ construction, we say that there are $k$ \emph{active layers} in round $\tau$ if the value of the counter in round $\tau$ is at most $2^k$ and no neuron in layer $j \geq k+1$ fired. Once we increase the counter, after at most $k+1$ rounds the value is updated. 
During this update operation, the network waits and ignores spikes from $x$ that might occur during this time window.
To implement this waiting mechanism, we introduce a \emph{Wait-Timer} ($WT$) module which uses the $\DetTimer^*$ module\footnote{Recall that $\DetTimer^*$ is a variant of the neural timer in which the time parameter is given as a soft-wired input and the upper bound on this input time is hard coded in the network.}. This $\DetTimer^*$ gets an input from $x_{ac}$ and the time parameter input $\bar{q}$ with $\log \ell$ neurons where $\ell = \log_2 \log_{\alpha} t$ is the number of layers in the module $AC$. The counting neurons $c_1, \ldots, c_{\ell}$ of $AC$ are connected to $\bar{q}$ as follows. Each $q_i$ has an incoming edge from $c_{2^{i-1}}$ with weight $w(c_{2^{i-1}},q_i)=1$ and threshold $b(q_i)=1$. Hence, the value of $\bar{q}$ is at least $k+1$ and at most $4k$ where $k$ is the number of active layers in $AC$. In order for the time parameter to stay stable throughout the update, for each $q_i$ we add a self loop with weight $w(q_i,q_i)=1$. The $WT$ module has two outputs, an inhibitor $g_r$ which fires as long as the timer did not finish the count, and an excitatory $g$ which fires after the count is over. We connect $r_r$ to $x_{ac}$ with weight $w(g_r,x_{ac})=-5$, preventing it from firing while the counter is not updated. We connect $g$ to an additional inhibitor neuron $q_r$ which inhibits the time parameter neurons $q_1, \ldots, q_\ell$ one round after we finished the count. The size of $WT$ is $O(\log \ell) = O(\log \log 1/\delta+ \log \log \log t )$.
\item{\textbf{Removing Assumption (S2):}} Two copies of the counting neurons $c_1, \ldots, c_\ell$ are introduced. The first copy $c'_1, \ldots c'_\ell$ allows us to copy the state of the counter $AC$ once its update proceess is complete. Each $c'_i$ has incoming edges from $c_i$ and the excitatory output of the $WT$ module, each with weight $w(g,c'_i)=w(c_i,c'_i)=1$ and threshold $b(c'_i)=2$. Thus, $c'_i$ fires iff in the previous round both $c_i$ and $g$ fired (implying that neuron $c_i$ was active when the counter finished the update). The second copy $c''_1, \ldots, c''_{\ell}$ holds the previous state of $AC$ during the update of the module. Each $c''_i$ has an incoming edge from $c'_i$ with weight $2$, a self loop with weight $1$, a negative edge from the inhibitor $g_r$ with the weight $(-1)$ and threshold $1$. Note that the inhibition of $c''_i$ occurs on the same round it receives the updated state from neuron $c'_i$.
Finally, the output layer $\bar{y}$ has incoming edges from neurons $c''_1, \ldots , c''_{\ell}$ instead of $c_1, \ldots , c_{\ell}$ with the same weights. 
\end{itemize}
Figure~\ref{fig:wait-timer} illustrates the modifications made to handle the general case. 

\begin{figure} 
	\begin{center}
\includegraphics[scale=0.3]{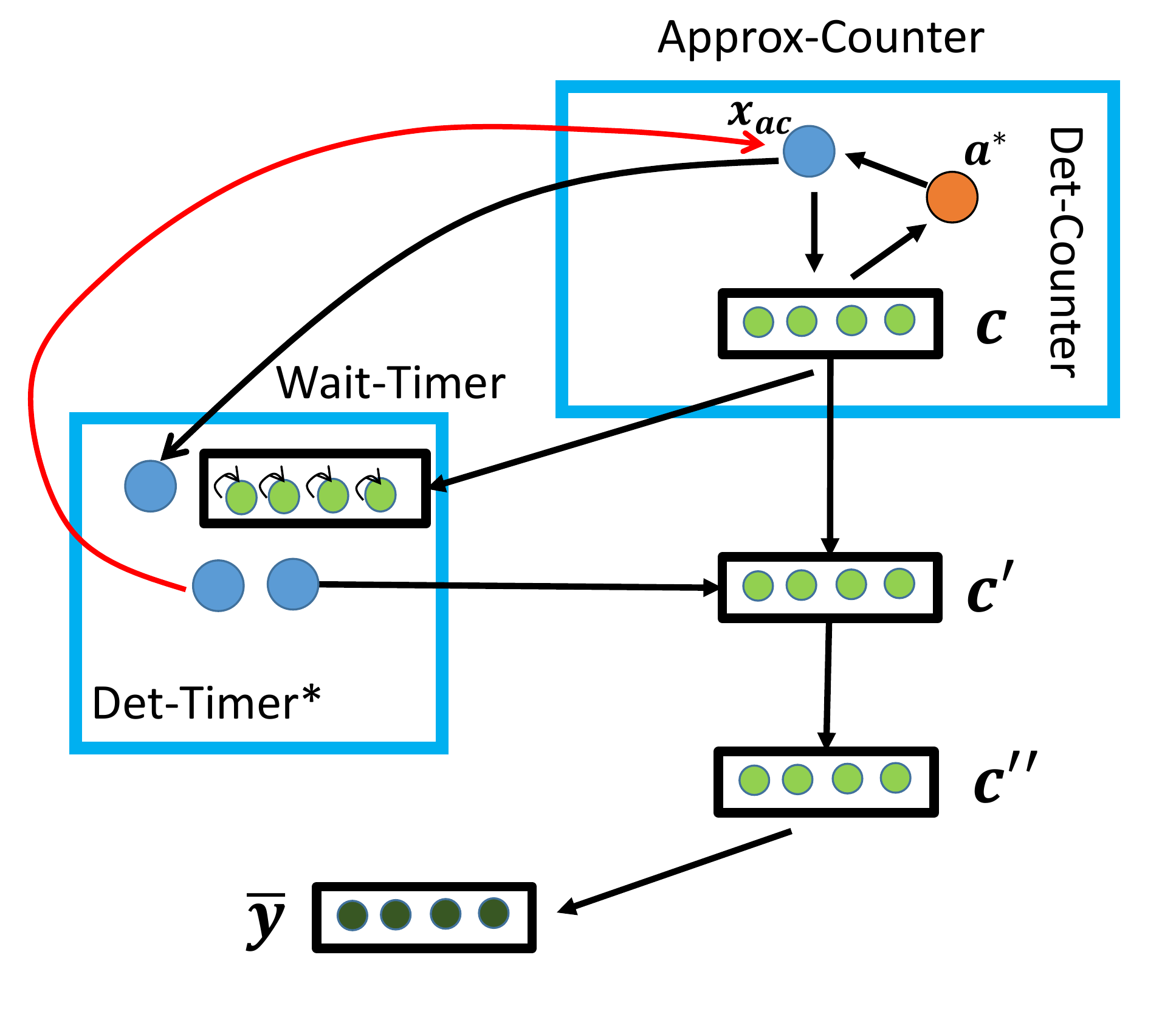}
\caption{A description of the modifications in $\ApproxCounter$ (to handle the general case). The \emph{Wait-Timer} ($WT$) module is implemented as a $\DetTimer^*$ with input from $x_{ac}$ and time input from the counting neurons of $AC$. 
The inhibitor output of $WT$ inhibits the input neuron $x_{ac}$, preventing it from firing during the update process of the $AC$ counter. We have two copies of the counting neurons of $AC$ denoted as $c'$ and $c''$. These copies are used for the output vector $\bar{y}$ to receive a correct input from $AC$ at all times, even during the update process of the $AC$ counter. Once the $WT$ module finishes its count, in order to copy the information from $c$ to $c''$, we use $c'$ as that OR gates between $c$ and the excitatory output of the module $WT$.	
 \label{fig:wait-timer}}
	\end{center}	
\end{figure} 
 
\paragraph{Proof of Thm.~\ref{thm:approx-counting} (for the general case).}
Assume $x$ fired $n$ times up to round $\tau$. If $n \leq s$ we count the number of times $x$ fired explicitly via the $SC$ module.
We note that in round $\tau$ the counter might be still updating the last $O(\log n)$ spikes of $x$. By the $\DetCounter$ construction, the value of the counter is at least $\frac{n-\log n}{2}= \Theta(n)$, and therefore we indeed output a constant approximation of $n$ with probability $1$.
 
Otherwise, $n \geq s$. First note that when we switch from the $SC$ to the $AC$ counter, we might omit at most $\Theta(\log 1/\delta)$ spikes due to the delay in the $\DetCounter$ module. Since $n \geq s = \Theta(1/\delta^2)$ this is negligible, as we want a constant approximation.
Next, we bound the number of times $x$ might have fired during the rounds in which the wait module $WT$ was active. 
As we only omit attempts to increase the counter, by Claim~\ref{clm:cheb} with probability at least $1 - \delta$, the value of counter has been increased for at most $\log_{\alpha} (2n(\alpha-1))$ times.

Each time that the counter value is increased, the waiting module $WT$ is active for at most $4\log_2 \log_{\alpha} 2n(\alpha-1) \leq 4 \log n$ rounds. Thus, in total we omit at most $4 \log n \cdot \log_{\alpha} (2n(\alpha-1)) < 4\sqrt{n}\log^2 n$ spiking events. In addition, since the copy neurons $c''_1, \dots, c''_{\ell}$ might hold the previous value of the counter in round $\tau$, we might lose another factor of two in the output layer. All together, in round $\tau$ the output $\bar{y}$ holds a constant approximation of $n$ and Theorem~\ref{thm:approx-counting} holds for the general case as well.

\section{Missing Details for Synchronizers}\label{sec:append-sync}

\subsection{Missing Details for the Asynchronous Analog of $\DetTimer$}\label{app:det-timer-async}

\paragraph{Proof of Lemma \ref{lem:det-timer-async}.}
The construction starts with $t'=t/2L$ layers of the $\DetTimer$ network. These layers are modified as follows (see Figure~\ref{fig:timer-async} for comparison with the standard construction).
\begin{itemize}
\item Neurons $a_{1,1}$ and $a_{1,2}$ are connected by a chain of length $4 L$. all neurons in the chain as well as $a_{1,2}$ have an incoming edge from the previous neuron in the chain with weight $1$ and threshold $1$. 
\item For every $i \geq 2$, the inhibitor neuron $d_i$ has an incoming edge only from $a_{i,2}$ with weight $w(a_{i,2}, d_i)=1$ and threshold $b(d_i)=1$.
\item For every $i \geq 1$, the neurons $a_{i-1,2}$ and $a_{i,1}$ are connected by a chain of length $L$, instead of a direct edge, where the weight of the edge from the end of the chain to $a_{i,1}$, is $1$.
\item The neuron $a_{\log t',2}$ is an excitatory (rather than  an inhibitory) neuron, and the output neuron $y$ has one incoming edge from $a_{\log t',2}$ with weight $w(a_{\log t',2}, y)=1$ and threshold $b(y)=1$.
\item A newly introduced inhibitor neuron $r$ that has an incoming edge from $a_{\log t',2}$ with weight $w(a_{\log t',2},r)=1$, threshold $b(r)=1$, and negative outgoing edges to all neurons in the timer with weight $-2$ for clean-up purpose. 
\end{itemize}
%
%
\begin{figure} 
\begin{center}
\includegraphics[width=0.5\textwidth]{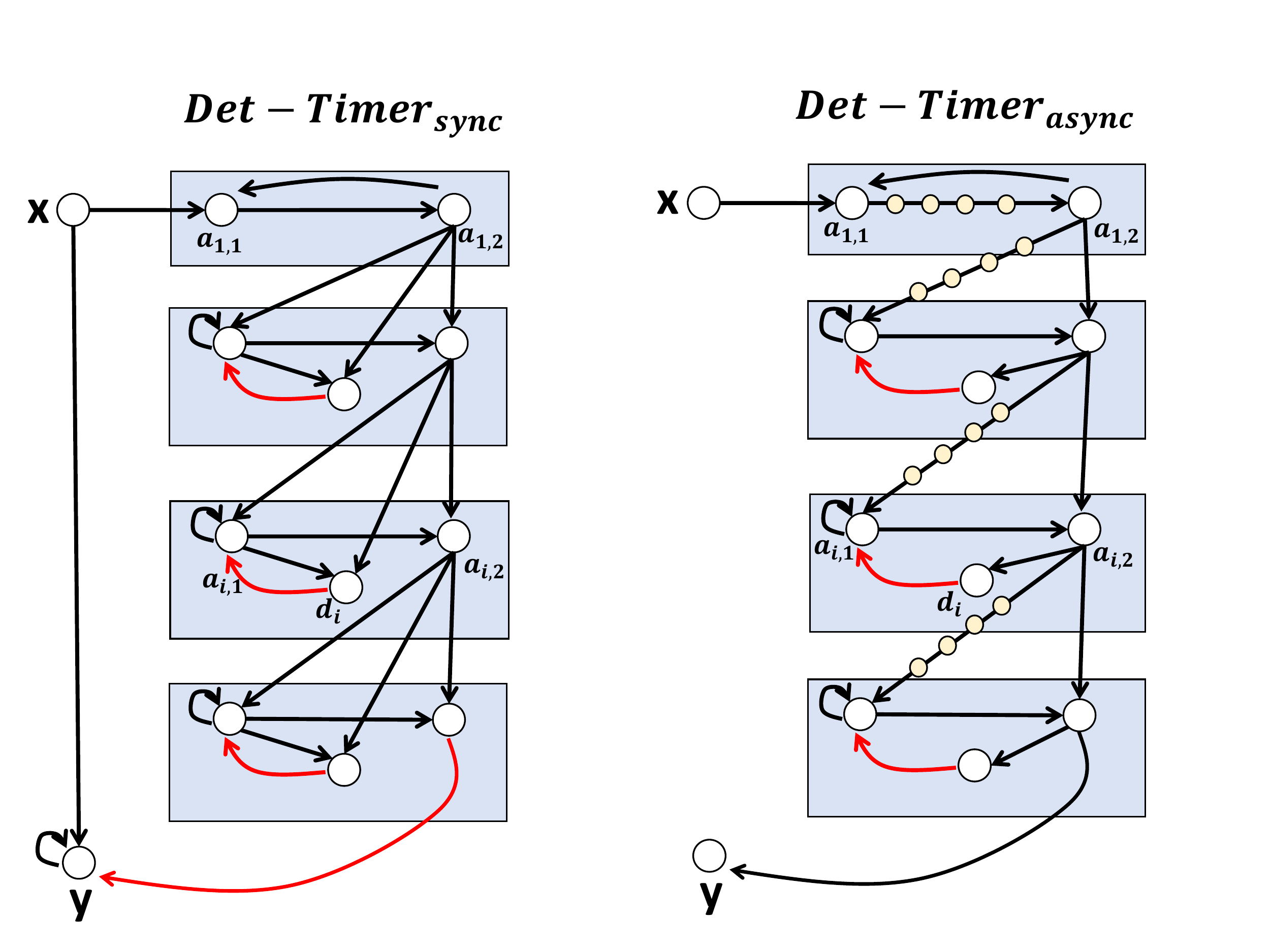}
\caption{$\DetTimer_{\async}$ versus $\DetTimer_{\sync}$. Left: The deterministic timer $\DetTimer_{\sync}$ network. Right: The modified $\DetTimer_{\async}$ network which works in the asynchronous setting. We add a chain of $L$ neurons in the first layer and between neurons $a_{i-1,2}$ and $a_{i,1}$, where $L$ is the upper bound on the response latency of a single edge in  the asynchronous setting.  \label{fig:timer-async}}
\end{center}
\vspace{-20pt}
\vspace{1pt}
\end{figure} 
The correctness is based on the following auxiliary claim.
\begin{claim} \label{clm: timer-async}
Fix a layer $i \geq 2$. Assume that (1) $a_{i-1,2}$ fired for the first time in round $f_{i-1}$, and that (2) it fires every $\tau_{i-1}$ rounds. It then holds that (1a) $a_{i,2}$ fires for the first time in round $f_i$ for $f_i \in [f_{i-1}+ \tau_{i-1}+1, f_{i-1}+ \tau_{i-1}+ L^2+L]$, and that (1b) $a_{i,2}$ fires from that point on for every $\tau_i \in [2\cdot \tau_{i-1}, 2\cdot \tau_{i-1}+(L^2+ L)]$ rounds.
\end{claim}
\begin{proof}
Assume that neuron $a_{i-1,2}$ fires every $\tau_{i-1}$ rounds starting round $f_{i-1}$.
It then holds that $a_{i,2}$ gets the spike from $a_{i-1,2}$ strictly \emph{before} the spike of $a_{i,1}$.
Specifically, it gets the spike from $a_{i-1,2}$ by round
$\tau\leq f_{i-1}+L$, and it receives the spike from $a_{i,1}$ in some round $\tau' \geq f_{i-1}+L+1$. 
Note that it is crucial that the spike from $a_{i-1,2}$ arrives earlier to $a_{i,2}$, as otherwise $a_{i,2}$ will fire in round $\tau$. 
As a result, the first time $a_{i,2}$ fires is after round $f_{i-1}+\tau_{i-1}+1$ and therefore $f_i \geq f_{i-1}+\tau_{i-1}+1$. 
Due to the self loop on $a_{i,1}$, neuron $a_{i,2}$ gets a spike from $a_{i,1}$ in every round $\tau''\geq \tau'$. 
Because the latencies are fixed, $a_{i,2}$ gets a signal from $a_{i-1,2}$ every $\tau_{i-1}$ rounds, and therefore $a_{i,2}$ fires by round $\tau'+\tau_{i-1}$. Since each edge has latency of at most $L$, it holds that $\tau' \leq f_{i-1} + L^2+L$, hence $f_i \leq f_{i-1} + L^2+L + \tau_{i-1}$ and (1a) follows. 

We now show (1b). We first observe that $a_{i,1}$ stops firing at least $L$ rounds \emph{before} the next firing of $a_{i-1,2}$.
This holds since once $a_{i,2}$ fires in round $f_i$, after at most $L$ rounds the inhibitor $d_i$ fires, and after at most $2L$ rounds neuron $a_{i,1}$ is inhibited. 
Since $\tau_j \geq 4L$ (due to the chain in the first layer), it indeed holds that in the next round when $a_{i-1,2}$ fires, no neuron in layer $i$ fires.  
Since the latency of each edge is fixed and $a_{i-1,2}$ fires every $\tau_{i-1}$ rounds by our assumption, we conclude that $a_{i,2}$ fires every $\tau_i$ rounds where $\tau_i \in [2 \tau_{i-1}, 2\tau_{i-1}+L^2+L]$.
\end{proof}

\begin{claim}
Assume that $x$ fired in round $\tau_0$. Then for every $i \geq 1$ it holds that:
(1) the neuron $a_{i,2}$ fires for the first time during the interval $[\tau_0+2^{i}\cdot 2L, \tau_0+2^{i}\cdot8L^2]$ and (2) it fires every $\tau_i$ rounds for $\tau_i \in [2^{i}\cdot 2L, 2^{i}\cdot(4L^2)]$. 
\end{claim}
%
\begin{proof}
Once the input neuron $x$ fired in round $\tau_0$, the neuron $a_{1,2}$ fires for the first time in round $f_1 \in [\tau_0 + 4L,\tau_0 + 4L^2 + L]$ and continue to fire every $\tau_1$ rounds for $\tau_1 \in [4L , 4L^2 +L]$. This is due to the chain between $a_{1,1}$ and $a_{1,2}$ and the fact that the latency $\ell(e)$ is fixed for every $e$.
Using Claim~\ref{clm: timer-async} in an inductive manner, we conclude that for every $i \geq 1$:
(1) $a_{i,2}$ fires every $\tau_i \in [2^i\cdot 2L, 2^i \cdot 4L^2]$ rounds,
(2) $a_{i,2}$ fires for the first time in round $f_i \in [\tau_0 + 2^{i}\cdot 2L, \tau_0 + 2^{i}\cdot 8L^2 ] $. 	
\end{proof}
Since the edge between neuron $a_{\log t',2}$ and the output neuron has latency at most $L$, we conclude that if the input neuron $x$ fires in round $\tau_0$, the output neuron fires in round $\tau \in [\tau_0+2Lt', \tau_0+9L^2t']$. Because $t'=t/2L$, given that the input $x$ fired in round $\tau$, the output neuron fires between round $\tau+t$ and round $\tau+5Lt$ and Lemma \ref{lem:det-timer-async} follows.

	%
\end{document}